\documentclass[10pt,journal,compsoc]{IEEEtran}
\usepackage[pdftex]{graphicx}
\usepackage{subfigure}
\usepackage{amsmath,amssymb,amsfonts} % define this before the line numbering.
\usepackage{color}
\usepackage{blindtext}
\usepackage{algorithm}
\usepackage{algorithmic}
\usepackage{cite}  %incompitable with biblatex
\usepackage{amsthm}
\usepackage[normalem]{ulem}
\usepackage{url}
\usepackage{comment}
\usepackage{multirow}
\usepackage{diagbox}
\usepackage{indentfirst}
\usepackage{url}

\newtheorem{thm}{Theorem}
\newtheorem{lemma}[thm]{Lemma}
\DeclareMathOperator*{\argmin}{\mathrm{argmin}}

\begin{document}

%macro for raising the point in decimal numbers; see example in the abstract
\newcommand{\point}{
    \raise0.7ex\hbox{.}
    }

%Do   -- NOT --    use any additional macros
%===========================================================
\title{Kernelized LRR on Grassmann Manifolds for Subspace Clustering} % Replace with your title

\author{Boyue~Wang, %~\IEEEmembership{Member,~IEEE,}
        Yongli~Hu~\IEEEmembership{Member,~IEEE,} Junbin~Gao, Yanfeng~Sun~\IEEEmembership{Member,~IEEE,} and Baocai~Yin % stops a space
\IEEEcompsocitemizethanks{\IEEEcompsocthanksitem  Boyue Wang, Yongli Hu, Yanfeng Sun and Baocai Yin are with Beijing Municipal Key Lab of Multimedia and Intelligent Software Technology, College of Metropolitan Transportation, Beijing University of Technology, Beijing 100124, China. 
E-mail: boyue.wang@emails.bjut.edu.cn, \{huyongli,yfsun,ybc\}@bjut.edu.cn \protect
\IEEEcompsocthanksitem Junbin Gao is with Discipline of Business Analytics, The University of Sydney Business School, The University of Sydney, Camperdown, NSW 2006,  Australia. \protect E-mail: junbin.gao@sydney.edu.au}% 
}

\maketitle

%===========================================================
%\IEEEcompsoctitleabstractindextext{%
\begin{abstract}
Low rank representation (LRR) has recently attracted great interest due to its pleasing efficacy in exploring low-dimensional subspace structures embedded in data. One of its successful applications is subspace clustering, by which data are clustered according to the subspaces they belong to. In this paper, at a higher level, we intend to cluster subspaces into classes of subspaces. This is naturally described as a clustering problem on Grassmann manifold. The novelty of this paper is to generalize LRR on Euclidean space onto an LRR model on Grassmann manifold in a uniform kernelized LRR framework. The new method has many applications in data analysis in computer vision tasks. %Several clustering experiments are conducted on Face or expression image sets, Large scale object image sets, Human action datasets and Traffic scene video clip sets. 
The proposed models have been evaluated on a number of practical data analysis applications. The experimental results show that the proposed models outperform a number of state-of-the-art subspace clustering methods.
\end{abstract}

\ifCLASSOPTIONpeerreview
 \begin{center} \bfseries EDICS Category: SMR-Rep\end{center}
\fi
 
\IEEEpeerreviewmaketitle

\section{Introduction}\label{Sec:1}

In recent years, subspace clustering or segmentation has attracted great interest in image analysis, computer vision, pattern recognition and signal processing \cite{XuWunsch-II2005,ElhamifarVidal2013}. The basic idea of subspace clustering is based on the fact that most data often have intrinsic subspace structures and can be regarded as the samples of a union of multiple subspaces. Thus the main goal of subspace clustering is to group data into different clusters, data points in each of which justly come from one subspace. To investigate and represent the underlying subspace structure, many subspace methods have been proposed, such as the conventional iterative methods \cite{Tseng2000}, the statistical methods \cite{GruberWeiss2004,TippingBishop1999a, ,HoYangLimLeeKriegman2003}, 
the factorization-based algebraic approaches \cite{Kanatani2001,MaYangDerksenFossum2008}, and the spectral clustering-based methods \cite{Luxburg2007,LiuLinSunYuMa2013,ChenLerman2009,ElhamifarVidal2013,LangLiuYuYan2012}. %LiuYan2011,ChenLerman2009,ElhamifarVidal2013,LangLiuYuYan2012,FavaroVidalRavichandran2011
These methods have been successfully applied in many applications, such as image representation \cite{HongWrightHuangMa2006}, motion segement \cite{Kanatani2001}, face classification \cite{LiuYan2011} and saliency detection \cite{LangLiuYuYan2012}, etc.

Among all the subspace clustering methods aforementioned, the spectral clustering methods based on affinity matrix are considered having good prospects \cite{ElhamifarVidal2013}, in which an affinity matrix is firstly learned from the given data and then the final clustering results are obtained  by spectral clustering algorithms such as Normalized Cuts (NCut) \cite{ShiMalik2000} or simply the K-means. The key ingredient  in a spectral clustering method is to construct a proper affinity matrix for data. In the typical method, Sparse Subspace Clustering (SSC) \cite{ElhamifarVidal2013}, one assumes that the data of subspaces are independent and are sparsely represented under the so-called $\ell_1$ Subspace Detection Property \cite{Donoho2004}, in which the within-class affinities are sparse and the between-class affinities are all zeros. It has been proved that under certain conditions the multiple subspace structures can be exactly recovered via $\ell_p  (p\leq 1)$ minimization \cite{LermanZhang2011}. 

In most of current sparse subspace methods, one mainly focuses on independent sparse representation for data objects.
However, the relation among data objects or the underlying global structure of subspaces that generate the subsets of data to be grouped is usually not well considered, while these intrinsic properties are very important for clustering applications. So some researchers explore these intrinsic properties and relations among data objects and then revise the sparse representation model to represent these properties by introducing extra constraints, such as the low rank constraint \cite{LiuLinSunYuMa2013}, the data Laplace consistence regularization \cite{LiuChenZhangXu2014} and the data sequential property \cite{TierneyGaoGuo2014}. etc. In these constraints, the holistic constraints such as the low rank or nuclear norm $\|\cdot\|_{*}$ are proposed in favour of structural sparsity. The Low Rank Representation (LRR) model \cite{LiuLinSunYuMa2013} is one of representatives. The LRR model tries to reveal the latent sparse property embedded in a data set in high dimensional space. It has been proved that, when the high-dimensional data set is actually from a union of several low dimension subspaces, the LRR model can reveal this structure through subspace clustering \cite{LiuLinSunYuMa2013}.

\begin{figure}
    \begin{center}
    \includegraphics[width=0.45\textwidth]{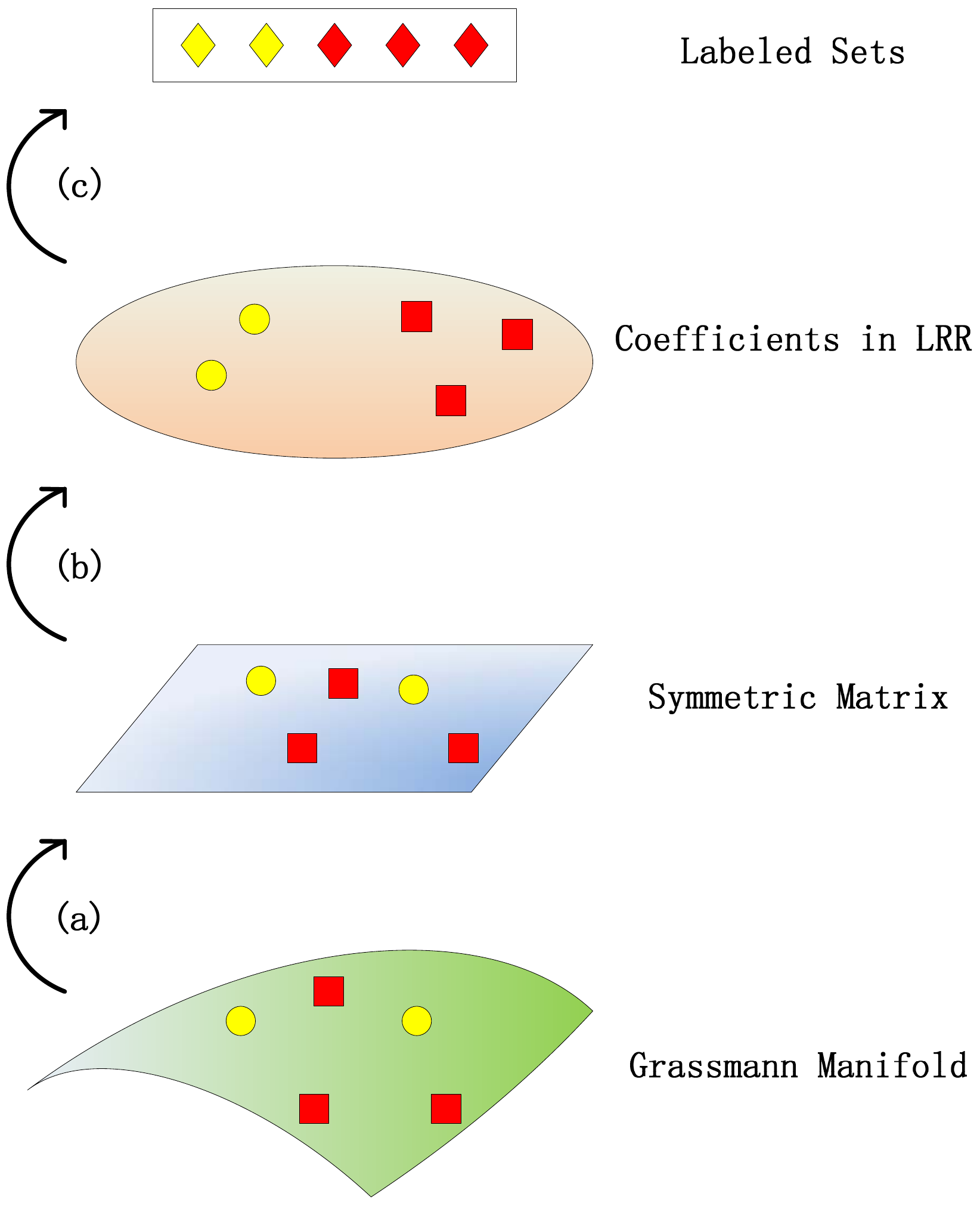}
    \end{center}
    \caption{An overview of our proposed LRR on Grassmann manifolds.  Three steps are involved in the proposed model: (a) The points on Grassmann manifold are mapped onto symmetric matrices. (b) LRR model is formulated in symmetric matrix space. (c) The coefficients in  LRR model are used by NCut for clustering.}\label{Fig1}
\end{figure}

Although most current subspace clustering methods show good performance in various applications, the similarity among data objects is measured in the original data domain. For example, the current LRR method is based on the principle of data self representation and the representation error is measured in terms of Euclidean alike distance. However, this hypothesis may not be always true for many high-dimensional data in practice where corrupted data may not reside in a linear space nicely. In fact, it has been proved that many high-dimensional data are embedded in low dimensional manifolds. For example, the human face images are considered as samples from a non-linear submanifold \cite{WangShanChenGao2008}. Generally manifolds can be considered as low dimensional smooth "surfaces" embedded in a higher dimensional Euclidean space. At each point of the manifold, manifold is locally similar to Euclidean space. To effectively cluster these high dimension data, it is desired to reveal the nonlinear manifold structure underlying these high-dimensional data and obtain a proper representation for the data objects.

There are two types of {manifold related learning tasks. In the so-called \textit{manifold learning}, one has to respect the local geometry existed in data but the manifold itself is unknown to learners. The classic representative algorithms for manifold learning include LLE (Locally Linear Embedding) \cite{RoweisSaul2000}, ISOMAP \cite{TenenbaumSilvaLangford2000}, LLP (Locally Linear Projection) \cite{HeNiyogi2003}, LE (Laplacian Embedding) \cite{BelkinNiyogi2001}, LTSA (Local Tangent Space Alignment) \cite{ZhangZha2004} and TKE (Twin Kernel Embedding) \cite{GuoGaoKwan2008}.

In the other type of learning tasks, we clearly know manifolds where the data come from. For example, in image analysis, people usually use covariance matrices of features as a region descriptor \cite{TuzelPorikliMeer2006}. In this case, such a descriptor is a point on the manifold of symmetrical positive definite matrices. More generally in computer vision, it is common to collect data on a known manifold. For example it is a common practice to use a subspace to represent a set of images \cite{TuragaVeeraraghavanChellappa2008}, while such a subspace is actually a point on the Grassmann manifold \cite{AbsilMahonySepulchre2004}. Thus an image set is regarded as a point from the known Grassmann manifold. This type of tasks incorporating manifold properties in learning is called \textit{learning on manifolds}. There are three major strategies in dealing with learning tasks on manifolds. 
\begin{enumerate}
\item \textit{Intrinsic Strategy:} The ideal but hardest strategy is to intrinsically perform learning tasks on manifolds based on their intrinsic geometry. Very few existing approaches adopt this strategy. 
\item \textit{Extrinsic Strategy:} The second strategy is to implement a learning algorithm within the tangent spaces of manifolds where all the linear relations can be exploited. In fact, this is a first order approximation to the Intrinsic strategy and most approaches fall in this category. 
\item \textit{Embedding Strategy:} The third strategy is to embed a manifold into a ``larger'' Euclidean space by an appropriate mapping like kernel methods and any learning algorithms will be implemented in this ``flatten'' embedding space. But for a practical learning task, how to incorporate the manifold properties of those known manifolds in kernel mapping design is still a challenging work.
\end{enumerate}

In this paper, we are concerned with the points on a particular \textit{known} manifold, the Grassmann manifold. We explore the LRR model to be used for clustering a set of data points on Grassmann manifold by adopting the aforementioned third strategy. %{\color{red}The intrinsic characteristics and geometry properties of Grassmann manifold will be exploited in algorithm design of LRR learning.} 
In fact, Grassmann manifold has a nice property that it can be embedded into the linear space of symmetric matrices  \cite{HarandiSalzmannJayasumanaHartleyLi2014,HarandiSandersonShenLovell2013}. By this way, all the abstract points (subspaces) on Grassmann manifold can be embedded into a Euclidean space where the classic LRR model can be applied. Then an LRR model can be constructed in the embedding space, where the error measure is simply taken as the Euclidean metric in the embedding space.
%{\color{red}This embedding property of Grassmann manifold is also used in \cite{HarandiSalzmannJayasumanaHartleyLi2014}}
%This idea can also be seen in the recent work \cite{HarandiSalzmannJayasumanaHartleyLi2014} for computer vision tasks. 
The main idea of our method is illuminated in Fig. \ref{Fig1}.

The contributions of this work are listed as follows:
\begin{itemize}\item Constructing an extended LRR model on Grassmann Manifold based on our prior work in \cite{WangHuGaoSunYin2014};
\item Giving the solutions and practical algorithms to the problems of the extended Grassmann LRR model under different noise models, particularly defined by  Frobenius norm and $\ell_2/\ell_1$ norm;
\item Presenting a new kernelized LRR model on Grassmann manifold.
\end{itemize}

The rest of the paper is organized as follows. In Section \ref{Sec:2}, we review some related works. In Section \ref{Sec:3}, the proposed LRR on Grassmann Manifold (GLRR) is described and the solutions to the GLRR models with different noises assumptions are given in detail. In Section \ref{Sec:4}, we introduce a general framework for the LRR model on Grassmann manifold from the kernelization point of view. In Section \ref{Sec:5}, the performance of the proposed methods is evaluated on several public databases. Finally, conclusions and suggestions for future work are provided in Section \ref{Sec:6}.
%-------------------------------------------------------------------------

\section{Related Works}\label{Sec:2}

In this section, we briefly review the existing sparse subspace clustering methods including the classic Sparse Subspace Clustering (SSC) and the Low Rank Representation (LRR) and then summarize the properties of Grassmann manifold that are related to the work presented in this paper.

\subsection{Sparse Subspace Clustering (SSC)}
Given a set of data drawn from a  union of unknown subspaces, the task of subspace clustering is to find the number of subspaces and their dimensions and bases, and then
segment the data set according to the subspaces. In recent years, sparse representation has been applied to subspace clustering, and the proposed Sparse Subspace Clustering (SSC) aims to find the sparse representation for the data set using $\ell_1$ regularization \cite{ElhamifarVidal2013}. The general SSC can be formulated as follows:
\begin{align}
\begin{aligned}
&\min\limits_{E,Z}\|E\|_{\ell}+\lambda\|Z\|_{1} \\  &\text{s.t.} \ \ Y=DZ+E,\; \text{diag}(Z)=0,
\end{aligned}\label{SR0}
\end{align}
where $Y\in \mathbb{R}^{d\times N}$ is a set of $N$ signals in dimension $d$ and $Z$ is the correspondent sparse representation of $Y$ under the dictionary $D$, and $E$ represents the error between the signals and its reconstructed values, which is measured by norm $|\cdot|_{\ell}$, particularly in terms of Euclidean norm, i.e., $\ell=2$ (or $\ell=F$) denoting the Frobenius norm to deal with the Gaussian noise, or $\ell=1$ (the Laplacian norm) to deal with the random gross corruptions or $\ell = \ell_2/\ell_1$ to deal with the sample-specific corruptions. Finally $\lambda>0$ is a penalty parameter to balance the sparse term and the reconstruction error.

In the above sparse model, it is critical to use an appropriate dictionary $D$ to represent signals. Generally, a dictionary can be learned from some training data by using one of many dictionary learning methods, such as the K-SVD method \cite{AharonEladBruckstein2006}. However, a dictionary learning procedure is usually time-consuming and so should be done in an offline manner. So many researchers adopt a simple and direct way to use the original signals themselves as the dictionary to find subspaces, which is known as the self-expressiveness property \cite{ElhamifarVidal2013}, i.e. each data point in a union of subspaces can be efficiently
reconstructed by a linear combination of other points in dataset. More specifically, every point in the dataset can be represented as a sparse linear combinations of other points from the same subspace.
Mathematically we write this sparse formulation as
\begin{align}
\begin{aligned}
&\min\limits_{E,Z}\|E\|_{\ell}+\lambda\|Z\|_1 \ \\ 
&\text{s.t.} \ \ Y=YZ+E,\; \text{diag}(Z)=0.
\end{aligned}\label{SR}
\end{align}
From the sparse representation matrix $Z$, an affinity matrix can be constructed. For example one commonly used form is $(|Z|+|Z^T|)/2$. This affinity matrix is interpreted as a graph upon which a clustering algorithm such as NCut is applied for final segmentation. This is the typical approach used in modern subspace clustering techniques.

\subsection{Low-Rank Representation (LRR)}
 The LRR can be regarded as one special type of sparse representation, in which rather than computing the sparse representation of each data point individually, the  global structure of data is collectively computed by the low rank representation of a set of data points. 
 
 The low rank measurement has long been utilized in matrix completion from corrupted or missing data \cite{CandesLiMaWright2011,WrightGaneshRaoPengMa2009}.
Specifically for clustering applications, it has been proved that, when a high-dimensional data set is actually composed of data from a union of several low dimension subspaces, 
 LRR model can reveal the subspaces structure underlying data samples \cite{LiuLinSunYuMa2013}. It is also proved that LRR has good clustering performance in dealing with the challenges in subspace clustering, such as the unclean data corrupted by noise or outliers, no prior knowledge
of the subspace parameters, and lacking of theoretical guarantees for the optimality of clustering methods \cite{LiuLinSunYuMa2013,LangLiuYuYan2012, ChengLiuWangHuangYan2011}.

The general LRR model can be formulated as the following optimization problem:
\begin{align}
\begin{aligned}
&\min\limits_{E,Z}\|E\|^2_{\ell}+\lambda\|Z\|_* \\ 
&\text{s.t.} \ \ Y=YZ+E, 
\end{aligned}\label{lrra}
\end{align}
where $Z$ is the low rank representation of the data set $Y$ by itself. Here the low rank constraint is achieved by approximating rank with the nuclear norm $\|\cdot\|_*$, which is defined as the sum of singular values of a matrix and is the low envelop of the rank function of matrices \cite{LiuLinSunYuMa2013}.
%\begin{equation}\label{lrrx}
%\min\limits_{E,Z}\|E\|^2_l+\lambda\|Z\|_* \ \ \text{s.t.} \ \ Y=YZ+E
%\end{equation}

Although the current LRR method has good performance in subspace clustering, it relies on Euclidean distance for measuring the similarity of the raw data. However, this measurement is not suitable to high-dimensional data with embedding low manifold structure. To characterize the local geometry of data on an \textit{unknown} manifold, the LapLRR methods \cite{LiuChenZhangXu2014,YinGaoLin2015} uses the graph Laplacian matrix derived from the data objects as a regularized term for the LRR model to represent the nonlinear structure of high dimensional data, while the reconstruction error of the revised model is still computed in Euclidean space.

\subsection{Grassmann Manifold}\label{Sec:2.3}
In recent years, Grassmann manifold has attracted great interest in computer vision research community. Although Grassmann manifold itself is quite abstract, it can be well represented as a matrix quotient manifold and its Riemannian geometry has been investigated for algorithmic computation \cite{AbsilMahonySepulchre2004}.

Grassmann manifold $\mathcal{G}(p,d)$ \cite{AbsilMahonySepulchre2004} is the space of all $p$-dimensional linear subspaces of $\mathbb R^d$ for $0\leq p\leq d$. A point on  Grassmann manifold is a $p$-dimensional subspace of $\mathbb R^d$ which can be represented by any orthonormal basis $X=[\mathbf x_1, \mathbf x_2, ..., \mathbf x_p]\in \mathbb R^{d\times p}$. The chosen orthonormal basis is called a representative of the subspace $\mathcal{S} = \text{span}(X)$.  Grassmann manifold $\mathcal{G}(p,d)$ has one-to-one correspondence to a quotient manifold of the Stiefel manifold on $\mathbb R^{d\times p}$, see \cite{AbsilMahonySepulchre2004}. 

Grassmann manifold has a nice property that it can be embedded into the space of symmetric matrices via a projection embedding, i.e. we can embed Grassmann manifold $\mathcal{G}(p,d)$ into the space of $d\times d$ symmetric positive semi-definite matrices $\text{Sym}_+(d)$  by the following mapping, see \cite{HarandiSandersonShenLovell2013},
% Harandi et.al\cite{HarandiSandersonShenLovell2013} define a mapping from each subspace to a symmetric matrix as the following form:
\begin{equation}\label{am1}
\begin{aligned}
\Pi: \mathcal{G}(p,d) \rightarrow \text{Sym}_+(d), \ \ \ \Pi(X) = XX^T.
\end{aligned}
\end{equation}

The embedding $\Pi(X)$ is diffeomorphism \cite{AbsilMahonySepulchre2004} (a one-to-one continuous and differentiable mapping with a continuous and differentiable inverse). Then it is reasonable to replace the distance on Grassmann manifold with the following distance defined on the symmetric matrix space under this mapping,
\begin{equation}\label{GDis}
\begin{aligned}
d_g(X_1,X_2)=\|\Pi(X_{1})-\Pi(X_{2})\|_F.
\end{aligned}
\end{equation}

This property was used in subspace analysis, learning and representation \cite{HammLee2008,SrivastavaKlassen2004, HarandiSandersonShiraziLovell2011}. 
The sparse coding and dictionary learning within the space of symmetric positive definite matrices have been investigated by using kerneling method \cite{HarandiSandersonShenLovell2013}. For clustering applications, the mean shift method was discussed on Stiefel and Grassmann manifolds in \cite{CetingulVidal2009}. Recently, a new version of K-means method was proposed to cluster Grassmann points, which is constructed by a statistical modeling method\cite{TuragaVeeraraghavanSrivastavaChellappa2011}. These works try to expand the clustering methods within Euclidean space to more practical situations on nonlinear spaces. Along with this direction, we further explore the subspace clustering problems on Grassmann manifold and try to establish a novel and feasible LRR model on Grassmann manifold.

\section{LRR on Grassmann Manifolds}\label{Sec:3}

\subsection{LRR on Grassmann Manifolds}\label{SubSec:3.2}
%Most LRR models employ different norm (ie. $l_0$-norm,$l_1$-norm,$l_F$-norm,$l_{2,1}$-norm etc) to restrain the error term or add other constrain term behind the subproblem (\ref{lrrx}). These models significantly develop the LRR in Euclidean space, while similar problem receive comparative few attention in Manifold space. As is known, manifold learning reveal intrinsic of the data from the observed phenomenon and find the inherent law of generated data. So we measure the distance of the points in Grassmann manifold space, not in Euclidean space. The subproblem (\ref{lrrx}) can be changed as
In the current LRR model \eqref{lrra}, the data reconstruction error is generally computed in the original data domain. For example, the common form of the reconstruction error is Frobenius norm, i.e. the error term can be chosen as follows, 
\begin{equation}\label{LRRERROR1eq}
\begin{aligned}
\|E\|_F^2 = \|Y-YZ\|_F^2 = \sum\limits_{i=1}^{N}\|y_i-\sum\limits_{j=1}^{N}z_{ji}y_j\|_F^2,
\end{aligned}
\end{equation}
where data matrix $Y = [y_1, y_2, ..., y_N]\in \mathbb{R}^{D\times N}$. % and $\|y_i - \sum\limits_{j=1}^{N}z_{ij}y_i\|_F^2$ presents the Euclidean distance between $y_i$ and its linear combination of all  other data points of $Y$.

As mentioned above, many high dimensional data have their intrinsic manifold structures. To extend an LRR model for manifold-valued data, two issues have to be resolved, i.e., (i) model error should be measured in terms of manifold geometry, and (ii) the linear relationship has to be re-interpreted. This is because the linear relation defined by $Y=YZ+E$ in \eqref{lrra} is no longer valid on a manifold.   %In these cases, the error should be measured according to the manifold geometry. %\footnote{As the manifold is generally no longer linear, so the linear combination on the manifold should be implemented via exp and log operations on the manifold. We ignore this for the simplicity of presenting our idea.}.
%So we consider signal representation for the data with manifold structure and employ an error measurement in LRR model based on the distance defined on manifold spaces \cite{WangSaligramaCastanon2011}.
 
In the extrinsic strategy mentioned in Section \ref{Sec:1}, one gets around this difficulty by using the Log map on a manifold to lift points (data) on a manifold onto the tangent space at a data point. This idea has been applied for clustering and dimensionality reduction on manifolds in \cite{GohVidal2008,WangHuGaoSunYin2015} and recently for LRR on Stiefel and SPD manifolds \cite{YinGaoGuo2015,FuGaoHongTien2015}.  

In this paper, instead of using the Log map tool, we extend the LRR model onto Grassmann manifold by using the \emph{Embedding Strategy}.  Given a set of Grassmann points $\{X_1, X_2, ..., X_N\}$ on Grassmann manifold $\mathcal{G}(p,d)$, we mimic the classical LRR defined in \eqref{lrra} and \eqref{LRRERROR1eq} as follows
\begin{equation}\label{ManifoldLRR1eq}
\begin{aligned}
\min\limits_{Z} \sum\limits_{i=1}^{N}\|X_i \ominus (\uplus_{j=1}^NX_j\odot z_{ji}) \|_{\mathcal{G}}+\lambda\|Z\|_*,
\end{aligned}
\end{equation}
where $\ominus$, $\uplus$ and $\odot$ are only dummy operators to be specified soon and $\|X_i \ominus (\uplus_{j=1}^NX_j\odot z_{ji}) \|_{\mathcal{G}}$ is to measure the error between the point $X_i$ and its ``reconstruction'' $\uplus_{j=1}^NX_j\odot z_{ji}$. Thus, to get an LRR model on Grassmann manifold, we should define proper distance and operators for the  manifold.

Based on the property of Grassmann manifold in \eqref{am1}, we have an easy way to use the distance of the embedded space to replace the manifold distance in the LRR model on Grassmann manifold as follows,
\[
\|X_i \ominus (\uplus_{j=1}^NX_j\odot z_{ji})\|_{\mathcal{G}} = d_g(X_i , (\uplus_{j=1}^NX_j\odot z_{ji})).
\]
This error measure not only avoids using Log map operator but also has simple computation with F-norm.   
\begin{figure}
    \begin{center}
    \includegraphics[width=0.95\linewidth]{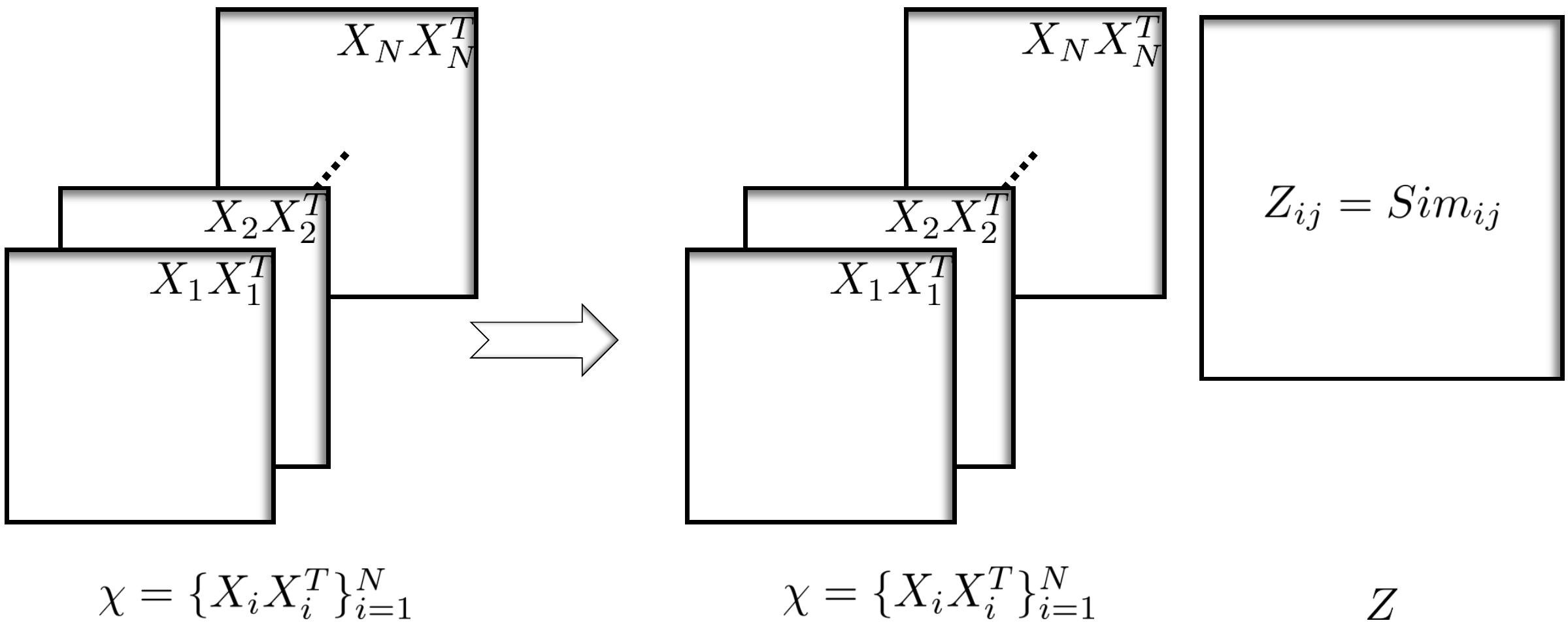}
    \end{center}
    \caption{The GLRR Model. The mapping of the points on Grassmann manifold, the tensor $\mathcal{X}$ with each slice being a symmetric matrix can be represented by the linear combination of itself. The element $z_{ij}$ of $Z$ represents the similarity between slices $i$ and $j$.} \label{LRRfig}
\end{figure}
%Before we propose the LRR model on Grassmann manifold, we first generalize the LRR model (\ref{lrra}) with F-norm as below
%Under the embedding mapping defined in \eqref{am1}, we intend to use the distance defined as \eqref{GDis} in the embedded symmetric matrix space as the measure on the Grassmann manifold. Thus we replace the manifold distance in \eqref{ManifoldLRR1eq} by $\|X_i \ominus (\uplus_{j=1}^NX_j\odot Z_{ij})\|_{\mathcal{G}} = d_g(X_i , (\uplus_{j=1}^NX_j\odot Z_{ij}))$.

Additionally, the mapping \eqref{am1} maps a Grassmann point to a point in the $d\times d$ symmetric positive semi-definite matrices space $\text{Sym}_{+}(d)$ in which there is a linear combination operation if the coefficients are restricted to be positive. So it is intuitive to replace the Grassmann points with its mapped points to implement the combination in \eqref{ManifoldLRR1eq}, i.e.
\[
\biguplus^N_{j=1}X_j\odot z_{ji}= \sum^N_{j=1} z_{ji} (X_jX^T_j),  \;\; \text{for }\; i = 1, 2, ..., N.
\]
Furthermore, we stack all the symmetric matrices $X_iX^T_i$'s as front slices of a 3rd order tensor $\mathcal{X}$, i.e., $\mathcal{X}(:,:,i) = X_iX^T_i$, then all the $N$ linear relations above can be simply written as $\mathcal{X}\times_3 Z$, where $\times_3$ means the mode-3 multiplication of a tensor and a matrix, see \cite{KoldaBader2009}. Thus the self-representation in \eqref{lrra} can be represented by 
\[
\mathcal{X} = \mathcal{X}\times_3 Z + \mathcal{E}
\]
where $\mathcal{E}$ is the error tensor. The representation is illustrated in Fig. \ref{LRRfig}.

In the following subsections, we give two LRR models on Grassmann manifold for two types of noise cases.

\subsubsection{LRR on Grassmann Manifold with Gaussian Noise (GLRR-F) \cite{WangHuGaoSunYin2014}}
For the completion of this paper, we include our prior work reported in the conference paper \cite{WangHuGaoSunYin2014}. This LRR model on Grassmann manifold, based on the error measurement defined in (\ref{GDis}), is defined as follows,
\begin{equation}\label{GLRR}
\begin{aligned}
%\min\limits_{E,Z}\|E\|_F^2+\lambda\|Z\|_* \ , \ \text{s.t.} \ \ \mathcal{X}=\mathcal{X}\times_{3}Z+E
&\min\limits_{\mathcal{E},Z}\|\mathcal{E}\|^2_F+\lambda\|Z\|_* \\
&\text{s.t.} \; \mathcal{X}=\mathcal{X} \times_3 Z+\mathcal{E}.
\end{aligned}
\end{equation}
%where $\chi$ is a set of points mapping from Grassmann Manifold, denoted as $\chi=\{X_{i}X_i^T\}_{i=1}^N$. Error term $E$ is a third-order tensor and $||\cdot||_F$ is the Fr%obenius norm, defined as the square root of the sum of the absolute squares of all elements in E. $\times_3$ is a tensor symbol which denotes the 3-norm product of a tensor with a matrix. where $\mathcal{E}$ is the construction error in a tensor form similar to $\mathcal{X}$. 

The Frobenius norm here is adopted because of the assumption that the model fits to Gaussian noise. We call this model the Frobenius norm constrained GLRR (GLRR-F). In this case, the error term in \eqref{GLRR} is
\begin{equation}\label{FULLERROR1eq}
\begin{aligned}
\|\mathcal{E}\|_F^2 = \sum\limits_{i=1}^{N}\|E(:,:,i)\|_F^2,
\end{aligned}
\end{equation}
where $E(:,:,i)=X_{i}X_i^T-\sum\limits_{j=1}^{N}z_{ij}(X_{j}X_{j}^T)$ is the $i$-th slice of $\mathcal{E}$, which is the error between the symmetric matrix $X_{i}X_i^T$ and its reconstruction of linear combination $\sum\limits_{j=1}^{N}z_{ji}(X_{j}X_{j}^T)$.

\subsubsection{LRR on Grassmann Manifold with $\ell_2/\ell_1$ Noise (GLRR-21)}
When there exist outliers in the data set, the Gaussian noise model is no longer a favored choice. Therefore, we propose using the so-called $\|\cdot\|_{\ell_2/\ell_1}$ noise model, which is used to cope with signal oriented gross errors in LRR clustering applications \cite{LiuLinSunYuMa2013}. Similar to the above GLRR-F model, we formulate the  $\|\cdot\|_{\ell_2/\ell_1}$ norm constrained GLRR model (GLRR-21) as follows,
\begin{equation}\label{GLRR_21}
\begin{aligned}
&\min\limits_{\mathcal{E},Z}\|\mathcal{E}\|_{\ell_2/\ell_1}+\lambda\|Z\|_* \ \\ 
&\text{s.t.} \ \ \mathcal{X}=\mathcal{X} \times_3 Z+\mathcal{E},
\end{aligned}
\end{equation}
where the $\|\mathcal{E}\|_{\ell_2/\ell_1}$ norm of a tensor is defined as the sum of the Frobenius norm of 3-mode slices as follows:
\begin{equation}\label{FULLERROR1eq21}
\|\mathcal{E}\|_{\ell_2/\ell_1} = \sum\limits_{i=1}^{N}\|E(:,:,i)\|_F.
\end{equation}
Note that \eqref{FULLERROR1eq21} without squares is different from \eqref{FULLERROR1eq}.
%With the third-order tensor E, calculating the low-rank representation matrix Z is a huge problem. For example, the mode-3 matricization of a tensor $\chi=\{X_{i}X_{i}^T\}_{i=1}^N\epsilon R^{d\times d\times N}$ is denoted by matrix $\chi_{(3)}\epsilon R^{N\times (d*d)}$, however $d*d$ is always too large to manage an LRR algorithm in such a scale. Fortunately, we find the scale can be break down into a small scale problem by decomposing the error E. Next subsection, we will introduce it detail.

\subsection{Algorithms for LRR on Grassmann Manifold}\label{SubSec:3.3}
The GLRR models in (\ref{GLRR}) and (\ref{GLRR_21}) present two typical optimization problems. In this subsection, we propose appropriate algorithms to solve them.

The GLLR-F model was proposed in our earlier ACCV paper \cite{WangHuGaoSunYin2014} where an algorithm based on ADMM was proposed. In this paper, we provide an even faster closed form solution for \eqref{GLRR} and further investigate the tensor structure in these models to obtain a practical solution for \eqref{GLRR_21}.

Intuitively, the tensor calculation can be converted to matrix operation by tensorial matricization, see \cite{KoldaBader2009}. For example, we can matricize the tensor $\mathcal{X}\in\mathbb{R}^{d\times d\times N}$  in mode-3 and obtain a matrix $\mathcal{X}_{(3)}\in \mathbb R^{N\times (d*d)}$ of $N$ data points (in rows). So it seems that the problem has been solved using the method of the standard LRR model. However, as the dimension $d*d$ is often too large in practical problems, the existing LRR algorithm could break down. To avoid this matter, we carefully analyze the representation of the construction tensor error terms and convert the optimization problems to its equivalent and readily solvable optimization model. In the following two subsections, we will give the detail of these solutions.

\subsubsection{Algorithm for the Frobenius Norm Constrained GLRR Model}\label{SubsubSec:3.3.1}

We follow the notation used in \cite{WangHuGaoSunYin2014}. By using variable elimination, we can convert problem \eqref{GLRR} into the following problem
\begin{align}
\min_Z \|\mathcal{X} -\mathcal{X}\times_3 Z\|^2_F + \lambda \|Z\|_*.  \label{newGLRR}
\end{align}
We note that $(X_j^{T}X_i)$ has a small dimension $p\times p$ which is easy to handle. %To simplify expression of the objective function \eqref{GLRR}, 
Denote
\begin{equation}\label{Delta_ij}
\Delta_{ij}=\text{tr}\left[(X_j^{T}X_i)(X_i^{T}X_j)\right],
\end{equation}
and the  $N\times N$ symmetric matrix 
\begin{equation} \label{MDeilta}\Delta=\left[\Delta_{ij}\right].
\end{equation}
%So we define an 
%\begin{equation} \label{MDeilta}
%\Delta=\left[\Delta_{ij}\right]_{i,j}.
%\end{equation}
\begin{comment}
Substituting (\ref{Delta_ij}) and (\ref{MDeilta}) into (\ref{FULLERROR3eq}), we have
\begin{equation}\label{EiDelta}
\begin{aligned}
\|E(:,:,i)\|_F^2 &=p-2\sum\limits_{j=1}^{N}z_{ij}\Delta_{ij}+\sum\limits_{j_1=1}^{N}\sum\limits_{j_2=1}^{N}z_{ij_1}z_{ij_2}\Delta_{j_{1}j_{2}} \\
                 &=p-2\sum\limits_{j=1}^{N}z_{ij}\Delta_{ij}+\mathbf z_{i}\Delta \mathbf z_{i}^T
\end{aligned}
\end{equation}
where $\mathbf z_i$ is the $i$-th row of $Z$. Substituting (\ref{EiDelta}) into (\ref{FULLERROR1eq}), we have a simplified reconstruction error representation:
\begin{equation}\label{E_rev0}
\begin{aligned}
\|E_{\mathcal{G}}\|_F^2 &= %\sum\limits_{i=1}^{N}\|E(:,:,i)\|_F^2 \\
                              %Np-2\sum\limits_{i=1}^{N}\sum\limits_{j=1}^{N}z_{ij}\Delta_{ij}+\sum\limits_{i=1}^{N}\mathbf z_{i}\Delta \mathbf z_{i}^T \\
                              Np-2\text{tr}[Z\Delta]+\text{tr}[Z\Delta Z^T]
\end{aligned}
\end{equation}
\end{comment}
Then we have the following Lemma.

\begin{lemma} Given a set of matrices $\{X_1, X_2, ...,$ $X_N\} \  s.t.  \  X_{i}\in R^{d\times p}  \  and \   X_i^{T}X_i=I$, if $\Delta = [\Delta_{ij}]_{i,j}\in R^{N\times N}$ with element $\Delta_{ij}=\text{tr}\left[(X_j^{T}X_i)(X_i^{T}X_j)\right]$, then the matrix $\Delta$ is semi-positive definite.
\end{lemma}
\begin{proof}Please refer to \cite{WangHuGaoSunYin2014}.
\end{proof}

\begin{comment}
\begin{proof} Denote by $B_i = X_iX^T_i$. Then $B_i$ is a symmetric matrix of size $d\times d$. Then
\begin{align*}
\Delta_{ij} & = \text{tr}\left[(X_j^{T}X_i)(X_i^{T}X_j)\right] = \text{tr}\left[(X_jX_j^{T})(X_iX_i^{T})\right]\\
&= \text{tr}[B_j B_i] = \text{tr}[B_j B^T_i] = \text{tr}[B^T_i B_j]\\
& = \text{vec}(B_i)^T \text{vec}(B_j),
\end{align*}
where $\text{vec}(\cdot)$ is the vectorization of a matrix.

Define a matrix $B = [\text{vec}(B_1), \text{vec}(B_2), ..., \text{vec}(B_N)]$. Then it is easy to show that
\[
\Delta = [\Delta_{ij}]_{i,j} = [\text{vec}(B_i)^T \text{vec}(B_j)]^N_{i,j=1} = B^T B.
\]
So $\Delta$ is a semi-positive definite matrix. %and by taking $L = B^T$ we have completed the proof of Lemma.
\end{proof}
\end{comment}

From Lemma 1, we have the eigenvector decomposition for $\Delta$ defined by 
$
\Delta  = U D U^T,
$
where $U^TU = I$ and $D = \text{diag}(\sigma_i)$ with nonnegative eigenvalues $\sigma_i$.  Denote the square root of $\Delta$ by
$
\Delta^{\frac12} = UD^{\frac12}U^T,
$
then it is not hard to prove that problem \eqref{newGLRR} is equivalent to the following problem
\begin{align}
\min_Z \|Z\Delta^{\frac12} - \Delta^{\frac12}\|^2_F + \lambda \|Z\|_*. \label{simplifiedGLRR}
\end{align}

Finally we have
\begin{thm}Given that $\Delta = UDU^T$ as defined above, the solution to \eqref{simplifiedGLRR} is given by
\[
Z^* = U D_{\lambda}U^T,
\]
where $D_{\lambda}$ is a diagonal matrix with its $i$-th element defined by
\[
D_{\lambda}(i,i) = \begin{cases} 1 - \frac{\lambda}{\sigma_i}  & \text{ if } \sigma_i > \lambda, \\
0 & \text{ otherwise}.
\end{cases}
\]
\end{thm}
\begin{proof}
Please refer to the proof of Lemma 1 in \cite{FavaroVidalRavichandran2011}.
\end{proof}

According to Theorem 2, the main cost for solving the LRR on Grassmann manifold problem \eqref{GLRR} is (i) computation of the symmetric matrix $\Delta$ and (ii) a SVD for $\Delta$. This is a significant improvement to the algorithm presented in \cite{WangHuGaoSunYin2014}.

%-------------------------------------------------------------------------
\subsubsection{Algorithm for the $\ell_2/\ell_1$ Norm Constrained GLRR Model}\label{SubsubSec:3.3.2}

Now we turn to the GLRR-12 problem \eqref{GLRR_21}. Because the existence of $\ell_2/\ell_1$ norm in error term, the objective function  is not differentiable but convex. We propose using  the alternating direction method (ADM) method to solve this problem.

Firstly, we construct the following augmented Lagrangian function:
\begin{align}
L(\mathcal{E},Z,\xi)=&\|\mathcal{E}\|_{\ell_2/\ell_1}+\lambda\|Z\|_* +\langle \xi, \mathcal{X}-\mathcal{X} \times_3 Z-\mathcal{E} \rangle \notag\\
& +\frac{\mu}{2}\|\mathcal{X}-\mathcal{X} \times_3 Z-\mathcal{E}\|_F^2, \label{ALfun}
\end{align}
where $\langle \cdot, \cdot\rangle$ is the standard inner product of two tensors in the same order, $\xi$ is the Lagrange multiplier, and $\mu$ is the penalty parameter.

%Then ADM is used to decompose the minimization of $L$ w.r.t $\mathcal{E}$ and $Z$ simultaneously into two subproblems  w.r.t $\mathcal{E}$ and $Z$, respectively. More 
Specifically, the iteration of ADM for minimizing \eqref{ALfun} goes as follows:
\begin{align}
\mathcal{E}^{k+1}
=&  \argmin\limits_{\mathcal{E}} L(\mathcal{E},Z^k,\xi^k) \notag
\\
=&\argmin\limits_{\mathcal{E}} \|\mathcal{E}\|_{\ell_2/\ell_1} + \langle \xi^k, \mathcal{X}-\mathcal{X} \times_3 Z^k-\mathcal{E} \rangle \notag\\
&+\frac{\mu^k}{2} \|\mathcal{X}-\mathcal{X} \times_3 Z^k-\mathcal{E}\|_F^2,\label{ADM_E} \\
Z^{k+1} =&\argmin\limits_{Z} L(\mathcal{E}^{k+1},Z,\xi^k) \notag\\
=&\argmin\limits_{Z} \lambda \|Z\|_* +\langle \xi^k, \mathcal{X}-\mathcal{X} \times_3 Z-\mathcal{E}^{k+1} \rangle \notag\\
 &+\frac{\mu^k}{2}\|\mathcal{X}-\mathcal{X} \times_3 Z-\mathcal{E}^{k+1}\|_F^2, \label{ADM_Z}
 \\
\xi^{k+1} =&\ \xi^k+\mu^k[\mathcal{X}-\mathcal{X} \times_3 Z^{k+1}-\mathcal{E}^{k+1}], \label{ADM_xi}
\end{align}
where we have used an adaptive parameter $\mu^k$. The adaptive rule will be specified later in Algorithm 1.

The above ADM is appealing only if we can find closed form solutions to the subproblems \eqref{ADM_E} and \eqref{ADM_Z}.

Consider problem \eqref{ADM_E} first.  Denote $\mathcal{C}^k=\mathcal{X}-\mathcal{X} \times_3 Z^{k}$ and for any 3-order tensor $\mathcal{A}$ we use $A(i)$ to denote the $i$-th front slice $A(:,:,i)$ along the 3-mode as a shorten notation. Then we observe that \eqref{ADM_E} is separable in terms of matrix variable $E(i)$ as follows:
\begin{comment}
\begin{align}
\sum_{i=1}^N E^{k+1}(:,:,i)
&= \argmin\limits_{\mathcal{E}} \sum_{i=1}^N \|E(:,:,i)\|_F \notag\\
&\quad  + \sum_{i=1}^N \langle \xi^k(:,:,i), C^k(:,:,i)- E(:,:,i) \rangle \notag \\
&\quad  +\sum_{i=1}^N \frac{\mu}{2} \| C^k(:,:,i)- E(:,:,i)\|_F^2\label{ADM_E1}
\end{align}
So the problem in \eqref{ADM_E} can be separated as the subproblems over each slice, i.e.
\end{comment}
\begin{equation}\label{ADM_E2}
\begin{aligned}
E^{k+1}(i)
&= \argmin\limits_{E(i)}  \|E(i)\|_F  +  \langle \xi^k(i), C^k(i)- E(i) \rangle \\
&\quad  + \frac{\mu^k}{2} \| C^k(i)- E(i)\|_F^2\\
&= \argmin\limits_{E(i)}  \|E(i)\|_F  + \frac{\mu^k}{2} \| C^k(i)- E(i)+ \frac{1}{\mu^k} \xi^k(i)\|_F^2.\\
\end{aligned}
\end{equation}

From \cite{LiuLinSunYuMa2013}, we know that the problem in  \eqref{ADM_E2} has a closed form solution, given by
\begin{equation}\label{ADM_E2Slt}
\begin{aligned}
E^{k+1}(i)=
\begin{cases}
0& \text{if } M < \frac{1}{\mu^k};\\
(1-\frac{1}{M\mu^k})(C^k(i)+\frac{1}{\mu^k} \xi^k(i))& \text{otherwise}.
\end{cases}
\end{aligned}
\end{equation}
where $M=\|C^k(i)+\frac{1}{\mu^k} \xi^k(i)\|_F$.

Now consider problem \eqref{ADM_Z}. Denote
\begin{equation*}%\label{F1Z}
 f(Z)=\langle \xi^k, \mathcal{X}-\mathcal{X} \times_3 Z-\mathcal{E}^{k+1} \rangle +\frac{\mu^k}{2}\|\mathcal{X}-\mathcal{X} \times_3 Z-\mathcal{E}^{k+1}\|_F^2,
\end{equation*}
then problem \eqref{ADM_Z} becomes
\begin{align}
Z^{k+1} = \argmin\limits_{Z} \lambda \|Z\|_* + f(Z).  \label{ADM_Z_new}
\end{align}
We adopt the linearization method to solve the above problem. 
\begin{comment}
For this purpose, we need to compute $\partial f(Z)$ w.r.t. $Z$. Let us define
\begin{equation*} %\label{MPhi}
 \Phi^k =\left[\text{tr}(\xi^k(i)^T X_jX_j^T)\right]_{i,j},
\end{equation*}
and
\begin{equation*} %\label{MPsi}
 \Psi^k =\left[\text{tr}(E^{k+1}(i)^TX_jX_j^T)\right]_{i,j},
\end{equation*}
then after a tedious algebraic manipulation, we have 
\begin{equation*}%\label{P_F1Z}
\partial f(Z)= \mu^k Z\Delta -\mu^k \left(\Delta-\Psi^k+\frac1{\mu^k}\Phi^k\right)^T.
\end{equation*}
\end{comment}

For this purpose, we firstly utilize the matrices in each slice to compute the tensor operation in the definition of $f(Z)$. For the $i$-th slice of the first term in $f(Z)$, we have
\begin{equation*} %\label{F1Z1_Slc}
\begin{aligned}
 &\langle \xi^k(i), X_{i}X_i^T-\sum\limits_{j=1}^{N}z_{ji}X_{j}X_{j}^T-E^{k+1}(i) \rangle\\
% =&\text{tr}\left(\xi^k(:,:,i)^T( X_{i}X_i^T-\sum\limits_{j=1}^{N}z_{ij}X_{j}X_{j}^T-E_{\mathcal{G}}^{k+1}(:,:,i))\right) \\
 =&-\sum_{j=1}^N z_{ji} \text{tr}(\xi^k(i)^T X_jX_j^T) + \text{tr}(\xi^k(i)^T( X_{i}X_i^T-E^{k+1}(i))).
\end{aligned}
\end{equation*}
Define  a new matrix by
\begin{equation*} %\label{MPhi}
 \Phi^k =\left[\text{tr}(\xi^k(i)^T X_jX_j^T)\right]_{i,j},
\end{equation*}
then the first term in $f(Z)$ has the following representation:
\begin{equation}  \label{F1Z1}
 \langle \xi^k, \mathcal{X}-\mathcal{X} \times_3 Z-E^{k+1} \rangle =- \text{tr}(\Phi^kZ^T)+\text{const}.
\end{equation}
For the $i$-th slice of the second term of $f(Z)$, we have
\begin{equation*} %\label{F1Z2_Slc}
\begin{aligned}
 &\|X_{i}X_i^T-\sum\limits_{j=1}^{N}z_{ji}X_{j}X_{j}^T-E^{k+1}(i)\|_F^2 \\
%  =&\text{tr}\left((X_{i}X_i^T-\sum\limits_{j=1}^{N}z_{ij}X_{j}X_{j}^T-E_{\mathcal{G}}^{k+1}(:,:,i))^T \right.\\
%  &\quad\left.(X_{i}X_i^T-\sum\limits_{j=1}^{N}z_{ij}X_{j}X_{j}^T-E_{\mathcal{G}}^{k+1}(:,:,i))\right)\\
 =&\text{tr}((X_{i}X_i^T)^TX_{i}X_i^T)+\text{tr}(E^{k+1}(i)^TE^{k+1}(i)) \\
 &\quad +\sum_{j_1=1}^N\sum_{j_2=1}^N z_{j_1i}z_{j_2i}\text{tr}((X_{j_1}X_{j_1}^T)^T(X_{j_2}X_{j_2}^T))\\
 &\quad -2\text{tr}((X_{i}X_i^T)^TE^{k+1}(i))\\
 &-2\sum_{j=1}^N z_{ji}\text{tr}((X_jX_j^T)^T(X_{i}X_i^T-E^{k+1}(i))).
\end{aligned}
\end{equation*}
Denoting a matrix by
\begin{equation*} %\label{MPsi}
 \Psi^k =\left[\text{tr}(E^{k+1}(i)^TX_jX_j^T)\right]_{i,j}
\end{equation*}
and noting \eqref{MDeilta}, we will have
\begin{equation}\label{F1Z2}
\begin{aligned}
&\|\mathcal{X}-\mathcal{X} \times_3 Z-E^{k+1}\|_F^2\\
 =&\text{tr}(Z \Delta Z^T) -2\text{tr}((\Delta-\Psi^k)Z)+ \text{const.}
\end{aligned}
\end{equation}
Combining \eqref{F1Z1} and \eqref{F1Z2}, we have
\begin{equation*}%\label{F1Zfinal}
 f(Z)= \frac{\mu^k}{2}\text{tr}(Z\Delta Z^T) -\mu^k \text{tr}((\Delta-\Psi^k + \frac1{\mu^k}\Phi^k)Z)+\text{const.}
\end{equation*}
Thus we have
\begin{equation*}%\label{P_F1Z}
\partial f(Z)= \mu^k Z\Delta -\mu^k \left(\Delta-\Psi^k+\frac1{\mu^k}\Phi^k\right)^T.
\end{equation*}

Finally we can use the following linearized proximity approximation to replace \eqref{ADM_Z_new} as follows
\begin{align}
&Z^{k+1}\notag\\
=&\argmin_{Z}\lambda \|Z\|_* + \langle\partial f(Z^k), Z- Z^k\rangle + \frac{\eta \mu^k}2\|Z-Z^k\|^2_F\notag\\
=&\argmin_Z\lambda \|Z\|_* + \frac{\eta\mu^k}2\left\|Z-Z^k+\frac{\partial f(Z^k)}{\eta\mu^k}\right\|^2_F, \label{newProblem}
\end{align}
with a constant $\eta > \|\mathcal{X}\|^2$ where $\|\mathcal{X}\|^2$ is the matrix norm of the third mode matricization of the tensor $\mathcal{X}$. The new problem \eqref{newProblem} has a closed form solution given by, see \cite{CaiCandesShen2008},
\begin{align}
Z^{k+1} = U_z \mathcal{S}_{\frac{\lambda}{\eta\mu^k}}(\Sigma_z) V^T_z,\label{SolutionZ}
\end{align}
where $U_z\Sigma_zV^T_z$ is the SVD of $Z_k - \frac{\partial f(Z^k)}{\eta\mu^k}$ and $\mathcal{S}_{\tau}(\cdot)$ is the Singular Value Thresholding (SVT) operator defined by
\[
\mathcal{S}_{\tau}(\Sigma) = \text{diag}(\text{sgn}(\Sigma_{ii})(|\Sigma_{ii}| - \tau)).
\]

Finally the procedure of solving the $\ell_2/\ell_1$ norm constrained GLRR problem \eqref{GLRR_21} is summarized in Algorithm 1. For the purpose of the self-completion of the paper, we borrow the convergence analysis for Algorithm 1 from \cite{LinLiuSu2011} without proof.
\begin{thm}If $\mu^{k}$ is non-decreasing and upper bounded, $\eta > \|\mathcal{X}\|^2$, then the sequence $\{(Z^k, \mathcal{E}^k, \xi^k)\}$ generated by Algorithm 1 converges to a KKT point of problem \eqref{GLRR_21}.
\end{thm}

\begin{algorithm}\label{Algorithm1}
\renewcommand{\algorithmicrequire}{\textbf{Input:}}
\renewcommand\algorithmicensure {\textbf{Output:} }
\caption{ Low-Rank Representation on Grassmann Manifold with $\ell_2/\ell_1$ Norm Constraint.}
\begin{algorithmic}[1]
\REQUIRE The Grassmann sample set $\{X_i\}_{i=1}^N$,$X_i\in \mathcal{G}(p,d)$, the cluster number $k$ and the balancing parameter $\lambda$. \\
\ENSURE  The Low-Rank Representation $Z$ ~~\\
\STATE   Initialize:$Z^0=0$, $\mathcal{E}^0=\xi^0=0$, $\rho^0 = 1.9$, $\eta > \|\mathcal{X}\|^2$, $\mu^0=0.01$, $\mu_{\max}=10^{10}$, $\varepsilon_1=10^{-4}$ and $\varepsilon_2=10^{-4}$.
\STATE Prepare $\Delta$ according to \eqref{Delta_ij};
%\STATE   Computing $L$ by Cholesky Decomposition $\Delta = LL^{T}$;
\WHILE   {not converged}
\STATE   Update $\mathcal{E}^{k+1}$ according to \eqref{ADM_E2Slt};
\STATE   Update $Z^{k+1}$ according to \eqref{SolutionZ};
\STATE   Update $\xi^{k+1}$ according to \eqref{ADM_xi};
\STATE   Update $\mu^{k+1}$ according to the following rule:
         \[
         \mu^{k+1} \leftarrow \min\{\rho^k\mu^k,\mu_{\mbox{max}}\}
         \]
         where
         \[
         \rho^k = \begin{cases} \rho^0 & \text{if } \mu^k/\|\mathcal{X}\|\max\{\sqrt{\eta}\|Z^{k+1}-Z^k\|_F,\\
         &\phantom{\text{if }}\|\mathcal{E}^{k+1}-\mathcal{E}^k\|_F\} \leq \varepsilon_2\\
1 & \text{otherwise}
\end{cases}
         \]
\STATE   Check the convergence conditions:
\[
\|\mathcal{X} - \mathcal{X}\times_3 Z^{k+1} - \mathcal{E}^{k+1}\|/\|\mathcal{X}\| \leq \varepsilon_1
\]
and
\[
\mu^k/\|\mathcal{X}\|\max\{\sqrt{\eta}\|Z^{k+1}-Z^k\|_F, \|\mathcal{E}^{k+1}-\mathcal{E}^k\|_F\} \leq \varepsilon_2
\]
\ENDWHILE
\end{algorithmic}
\end{algorithm}

\section{Kernelized LRR on Grassmann Manifold}\label{Sec:4}
\subsection{Kernels on Grassmann Manifold}\label{SubSec:4.1}
In this section, we consider the kernelization of the GLRR-F model. In fact, the LRR model on Grassman manifold \eqref{GLRR} can be regarded a kernelized LRR with a kernel feature mapping $\Pi$ defined by \eqref{am1}. It is not surprised that $\Delta$ is semi-definite positive as it serves as a kernel matrix. It is natural to further generalize the GLRR-F based on kernel functions on Grassmann manifold.

There are a number of Grassmann kernel functions proposed in recent years in computer vision and machine learning communities, see \cite{WolfShashua2003,HarandiSandersonShiraziLovell2011,HarandiSalzmannJayasumanaHartleyLi2014,JayasumanaHartleySalzmannLiHarandi2014}. For simplicity, we focus on the following kernels:

\textit{1. Projection Kernel}:   This kernel is defined in \cite{HarandiSandersonShiraziLovell2011}. For any two Grassmann points $X_i$ and $X_j$, the kernel value is
\[
k^{\text{p}}(X_i, X_j) = \|X^T_i X_j\|^2_F = \text{tr}( (X_iX^T_i)^T(X_jX^T_j)).
\]
The feature mapping of the kernel is actually the mapping defined in \eqref{am1}.

\textit{2. Canonical Correlation Kernel}:  Referring to \cite{HarandiSandersonShiraziLovell2011}, this kernel is based on the cosine values of the so-called principal angle between two subspaces defined as follows
\begin{align*}
\cos(\theta_m) &= \max_{ \mathbf u_m\in\text{span}(X_i)}\max_{ \mathbf v_m\in\text{span}(X_j)} \mathbf u^T_m\mathbf v_m, \\
&\text{s.t. } \|\mathbf u_m\|_2 = \|\mathbf v_m\|_2 =1;\\
&\phantom{\text{s.t. }} {\mathbf u}^T_m\mathbf u_k = 0, \; k = 1, 2, ..., m-1;\\
&\phantom{\text{s.t. }} {\mathbf v}^T_m\mathbf v_l = 0, \; l = 1, 2, ..., m-1.
\end{align*}

We can use the largest canonical correlation value (the cosine of the first principal angle) as the kernel value as done in \cite{YamaguchiFukuiMaeda1998}, i.e.,
\[
k^{\text{cc}}(X_i, X_j) = \max_{ \mathbf x_i\in\text{span}(X_i)}\max_{ \mathbf x_j\in\text{span}(X_j)} \frac{\mathbf x^T_i\mathbf x_j}{\|\mathbf x_i\|_2\|\mathbf x_j\|_2}.
\]

The cosine of principal angles of two subspaces can be calculated by using SVD as discussed in \cite{BjorckGolub1973}, see Theorem 2.1 there.

Consider two subspaces $\text{span}(X_i)$ and $\text{span}(X_j)$ as two Grassmann points where $X_i$ and $X_j$ are given bases. If we take the following SVD
\[
X^T_iX_j = U \Sigma V^T,
\]
then the values on the diagonal matrix $\Sigma$ are the cosine values of all the principal angles. The kernel $k^{\text{cc}}(X_i, X_j)$ uses partial information regarding the two subspaces. To increase its performance in our LRR, in this paper, we use the sum of all the diagonal values of $\Sigma$ as the kernel value between $X_i$ and $X_j$. We still call this revised version the canonical correlation kernel.

\subsection{Kernelized LRR on Grassmann Manifold}\label{SubSec:4.2}
Let $k$ be any kernel function on Grassmann manifold. According to the kernel theory \cite{ScholkopfSmola2002}, there exists a feature mapping $\phi$ such that
\[
\phi: \mathcal{G}(p,n) \rightarrow \mathcal{F},
\]
where $\mathcal{F}$ is the relevant feature space under the given kernel $k$.

Give a set of points $\{X_1,X_2, ..., X_N\}$ on Grassmann manifold $\mathcal{G}(p,n)$, we define the following LRR model
\begin{align}
\min \|\phi(\mathcal{X}) - \phi(\mathcal{X})Z\|^2_{\mathcal{F}} + \lambda \|Z\|_*.
\label{KLRR}
\end{align}
We call the above model the Kernelized LRR on Grassman manifold, denoted by KGLRR, and KGLRR-cc, KGLRR-p for $k=k^{\text{cc}}$ and $k=k^{\text{p}}$, respectively. 

However, for KGLRR-p, the above model \eqref{KLRR} becomes the LRR model \eqref{newGLRR}.
Denote by $K$ the $N\times N$ kernel matrix over all the data points $X$'s. By using the similar derivation in \cite{WangHuGaoSunYin2014}, we can prove that the model \eqref{KLRR}  is equivalent to
\[
\min_Z - 2 \text{tr}(KZ) +  \text{tr}(Z K Z^T) + \lambda\|Z\|_*,
\]
which is equivalent to
\begin{align}
\min_Z \|ZK^{\frac12} - K^{\frac12}\|^2_F   + \lambda\|Z\|_*. \label{Eq:4October2014-1}
\end{align}
where $K^{\frac12}$ is the square root matrix of the kernel matrix $K$. So the Kernelized model KGLRR-p is similar to GLRR-F model in Section \ref{Sec:3}.

It has been proved that using multiple kernel functions improves performance in many application scenarios \cite{ScholkopfSmola2002}, %BachLanckrietJordan2004}, 
due to the virtues of different kernel functions for the complex data. So in practice, we can employ different kernel functions to implement the model in \eqref{KLRR}, even we can adopt a combined kernel function. For example, in our experiments, we use a combination of the above two kernel functions $k^{\text{cc}}$ and $k^{\text{proj}}$ as follows.
\[
k^{\text{ccp}}(X_i, X_j) = \alpha k^{\text{cc}}(X_i, X_j)+(1-\alpha)k^{\text{p}}(X_i, X_j).
\]
where $0<\alpha<1$ is a hand assigned combination coefficient. We denote the Kernelized LRR model of $k=k^{\text{ccp}}$ by KGLRR-ccp.

\subsection{Algorithm for KGLRR}\label{SubSec:4.3}
It is straightforward to use Theorem 2 to solve \eqref{Eq:4October2014-1}. For the sake of convenience, we present the algorithm below.

Let us take the eigenvector decomposition of the kernel matrix $K$
\[
K = UDU^T,
\]
where $D=\text{diag}(\sigma_1, \sigma_2, ...., \sigma_N)$ is the diagonal matrix of all the eigenvalues. Then the solution to \eqref{Eq:4October2014-1} is given by
\[
Z^* = UD_{\lambda}U^T,
\]
where $D_{\lambda}$ is the diagonal matrix with elements defined by
\[
D_{\lambda}(i,i) = \begin{cases} 1 - \frac{\lambda}{\sigma_i} & \text{if } \sigma_i > \lambda;\\
0 & \text{otherwise}.
\end{cases}
\]
This algorithm is valid for any kernel functions on Grassmann manifold.
%===========================================================
\section{Experiments}\label{Sec:5}
To evaluate the performance of the proposed methods, GLRR-21, GLRR-F/KGLRR-p and KGLRR-ccp, we select various public datasets of different types to conduct clustering experiments. These datasets are challenging for clustering applications. We divide these datasets into four types: 

\begin{itemize}
  \item \textbf{Face or expression image sets}, including the Extended Yale B face dataset  (\url{http://vision.ucsd.edu/content/yale-face-database}) and the BU-3DEF expression dataset (\url{http://www.cs.binghamton.edu/~lijun/Research/3DFE/3DFE_Analysis.html}).
  \item \textbf{Large scale object image sets}, including the Caltech 101 dataset (\url{http://www.vision.caltech.edu/feifeili/Datasets.htm}) and the ImageNet 2012 dataset (\url{http://www.image-net.org/download-images}).
  \item \textbf{Human action datasets}, including the Ballet dataset (\url{https://www.cs.sfu.ca/research/groups/VML/semilatent/}) and the SKIG dataset (\url{http://lshao.staff.shef.ac.uk/data/SheffieldKinectGesture.htm}).
   \item \textbf{Traffic scence video clip sets}, including the Highway Traffic Dataset (\url{http://www.svcl.ucsd.edu/projects/traffic/}) and a traffic road dataset we collected.
\end{itemize}

The proposed methods will be compared with the benchmark spectral clustering methods, Sparse Subspace Clustering (SSC) \cite{ElhamifarVidal2013} and Low-Rank Representation (LRR) \cite{LiuLinSunYuMa2013}, and several state-of-the-art clustering methods concerned with manifolds, including Statistical computations on Grassmann and Stiefel manifolds (SCGSM) \cite{TuragaVeeraraghavanSrivastavaChellappa2011}, Sparse Manifold Clustering and Embedding (SMCE) \cite{ElhamifarVidalNips2011} and Latent Space Sparse Subspace Clustering (LS3C) \cite{PatelNguyenVidal2013}.
In the sequel, we first describe the experiment setting, then report and analyze the clustering results on these datasets.

\subsection{Experiment Setting}

Our GLRR model is designed to cluster Grassmann points, which are subspaces instead of raw object/signal vectors (points). Thus before implementing the main components of GLRR and the spectral clustering algorithm (here we use Ncut algorithm), we must form subspaces from raw signals. Generally, a subspace can be  represented by an orthonormal basis, so we utilize the samples drawn from the same subspace to construct its orthonormal basis. Similar to the work in \cite{HarandiSandersonShiraziLovell2011,HarandiSandersonHartleyLovell2012}, we simply adopt SVD to construct subspace bases. Concretely, given a set of images, denoted by $\{Y_i\}_{i=1}^P$ with each $Y_i$ in size of $m \times n$ pixels, we construct a matrix $\Gamma=[\text{vec}(Y_1), \text{vec}(Y_2),..., \text{vec}(Y_P)]$ of size $(m*n)\times P$ by vectorizing all the images. Then $\Gamma$ is decomposed by SVD as $\Gamma=U\Sigma V$. We pick the first $p$ ($\leq P$) singular-vectors $X=[\mathbf u_1, \mathbf u_2, ..., \mathbf u_p]$ of $U$ to represent the entire image set as a point $X$ on the Grassmann  manifold $\mathcal{G}(p,m*n)$.

The setting of the model parameters affects the performance of our proposed methods. $\lambda$ is the most important penalty parameter for balancing the error term and the low-rank term in our proposed methods. Empirically, the value of $\lambda$ in different applications has big gaps, and the best value for $\lambda$ has to be chosen from a large range of values to get a better performance in a particular application. From our experiments, we have observed that when the cluster number is increasing, the best $\lambda$ is decreasing. Additionally, $\lambda$ will be smaller when the noise level in data is lower while $\lambda$ will become larger if the noise level higher. These observations are useful in selecting a proper $\lambda$ value for different datasets. The error tolerance $\varepsilon$ is also an important parameter in controlling the terminal condition, which bounds the allowed reconstructed error. We experimentally seek a proper value of $\varepsilon$ to make the iteration process stop at an appropriate level of reconstructed error.

For both SSC and LRR methods, which demand the vector form of inputs, the subspace form of points on Grassmann manifold cannot be used directly. So to compare our method with SSC and LRR, we have to vectorize each image set to construct inputs for SSC and LRR, i.e. we "vectorize" a set of images into a long vector by stacking all the vectors of the raw data in the image set in a particular order, e.g., in the frame order etc. However, in most of the experiments, we cannot simply take these long vectors because of high dimensionality for a larger image set.
In this case, we apply PCA to reduce these vectors to a low dimension which equals to either the dimension of subspace of Grassmann manifold or the number of PCA components retaining 95\% of its variance energy. Then PCA projected vectors will be taken as inputs for SSC and LRR.  

In our experiments, the performance of different algorithms is measured by the following  clustering accuracy  
\[
\text{Accuracy} = \frac{\text{number of correctly classified points}}{\text{total number of points}}\times 100\%.
\]

To clearly present our experiments, we denote by $C$ the number of clusters, $N$ the total number of image sets, $P$ the number of images in each image set and $p$ the dimension of subspace of a Grassmann point.

All the algorithms are coded in Matlab 2014a and implemented on an Intel Core i7-4770K 3.5GHz CPU machine with 32G RAM.

\subsection{Clustering on Face and Expression Image Sets}
Human face or expression image sets are widely used in computer vision and pattern recognition communities. They are considered as challenging data sets for either clustering or recognition applications. The main difficulty is that the face image is affected greatly by various factors, such as the complex structure of face, the non-rigid elastic deformation of expression, different poses and various light conditions. 
%Here we select the Extended Yale B dataset \cite{GeorghiadesBelhumeurKriegman2001} and the BU-3DEF expression dataset \cite{YinWeiSunWangRosato2006} to implement clustering experiments. 

\textit{1) Extended Yale B dataset} %\cite{GeorghiadesBelhumeurKriegman2001}}

Extended Yale B dataset contains face images of 38 individuals and each subject has about 64 frontal images captured in different illuminations. All the face images have been normalized to the size of $20\times 20$ pixels in 256 gray levels. Some samples of Extended Yale B dataset are shown in Fig. \ref{FigE1}.

\begin{figure}
    \begin{center}
    \includegraphics[width=0.45\textwidth]{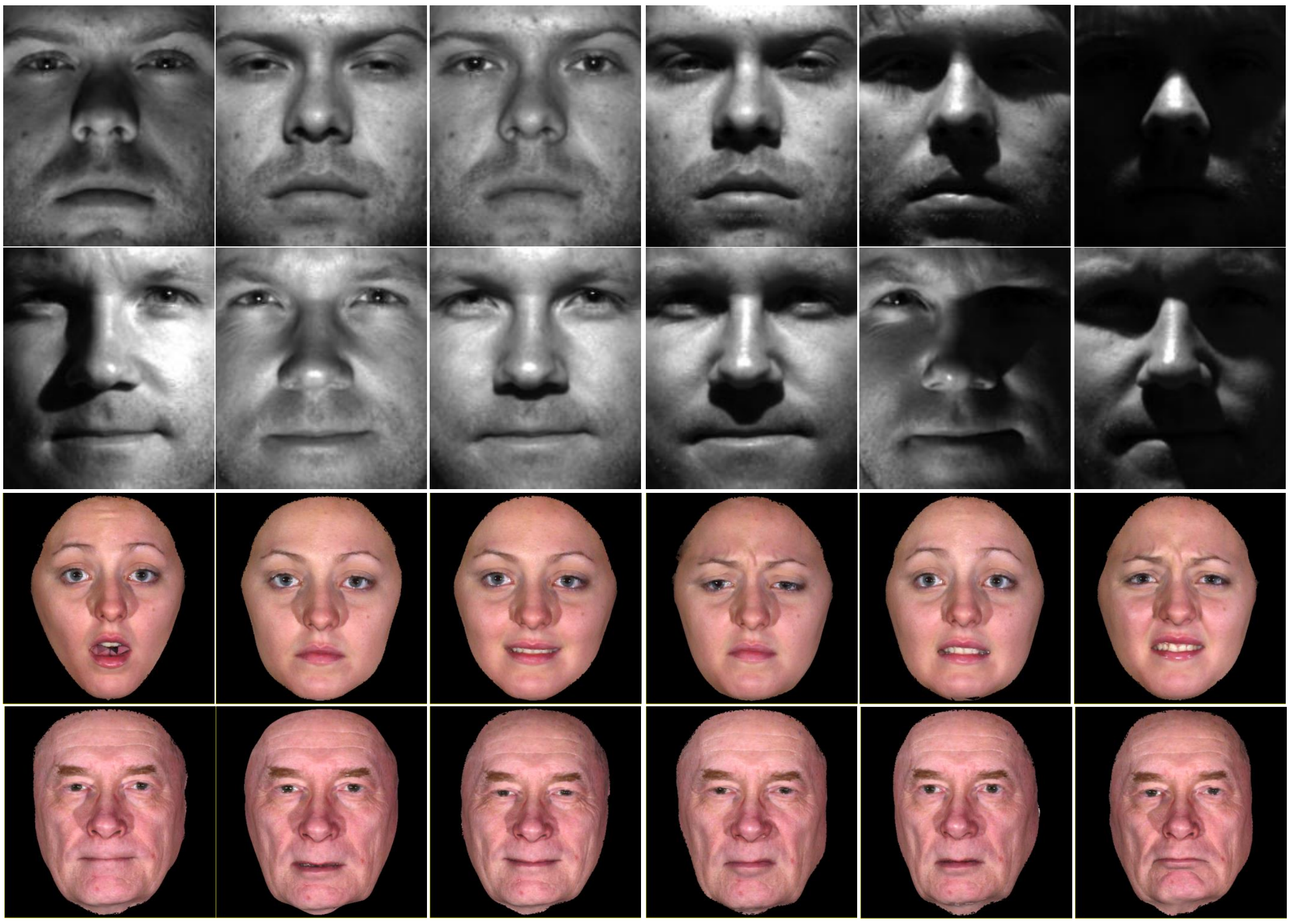}
    \end{center}
    \caption{Some samples from Extended Yale B dataset (the first two rows) and BU-3DEF dataset (the last two rows).}\label{FigE1}
\end{figure}

To prepare the experiment data, we randomly choose $P$ images from each subject to construct an image set and  $P$ is set to 4 or 8 in order to test the affection of different scales of image set for the clustering results. We produce 10 image sets for each subject, so there are totally 380 points for clustering. To get a Grassmann point, we use the aforementioned SVD operator to get the basis of subspace corresponding to each image set. The dimension of subspace $p=4$. Thus the Grassmann point $X \in \mathcal{G}(4,400)$ in this experiment. For SSC and LRR methods, the original vector of an image set has dimension of $20\times 20\times 4 = 1600$ or $20\times 20\times 8 = 3200$ for $P=4$ and $8$, respectively. Here, we reduce the dimension to $146$ by PCA. 

The experiment results are shown in Table \ref{Yaletab}. It shows that most experimental results with $P=8$ are obviously better than that with $P=4$. In fact, the larger $P$ is, the better the performance. When more images in the set, the impact of outlier images such as darkness faces or special expressions will be decreased. However a larger $P$ may also increase more variances to be fitted by the subspace.   
Compared with other manifold based methods, SCGSM, SMCE and LS3C, the excellent performance of our methods is due to the incorporation of low rank constraint over the similarity matrix $Z$. Finally we also note that the 
performance of LRR and SSC is greatly worse than all manifold based methods, which demonstrates incorporating manifold properties can also improve the performance in clustering.

\begin{table*}
  \centering
   \begin{tabular}{|c|c|c|c|c|c|c|c|c|}
     \hline
              Size of Image Sets &GLRR-F &GLRR-21 &KGLRR-ccp &LRR &SSC &SCGSM &SMCE  &LS3C  \\
     \hline
        4         & 0.6154 &0.5705 & \textbf{0.8283} & 0.3209 & 0.3526  & 0.4183 & 0.5978  & 0.4760\\
     \hline
        8         & 0.8878 &0.8526 & \textbf{0.8983} & 0.2788 & 0.3109  & 0.5657 & 0.8429  & 0.6250\\
     \hline
   \end{tabular}
  \caption{The clustering results on the Extend Yale B database with different setting of $P=4$ and $8$ for the size of image sets.}\label{Yaletab}
\end{table*}

\textit{2) BU-3DEF dataset} %\cite{YinWeiSunWangRosato2006}}

BU-3DEF dataset collects 3D face models and face images from 100 subjects of different genders, races and ages. Each subject has 25 face images, one neutral image and six expressions (happiness, disgust, fear, angry, surprise and sadness), each of which is  at four levels of intensity. 
In our experiment, we use these expression images for clustering. They are normalized and centered in a fixed size of $24\times 20$. Fig. \ref{FigE1} shows some samples of BU-3DEF dataset.

For each expression, we randomly select $P=6$ images to construct an image set and totally obtain 100 image sets. Then, a Grassmann point $X \in \mathcal{G}(4,400)$ is created for each image set by using SVD. There are $C=6$ classes of expressions for clustering. For SSC and LRR, the original vectors with dimension $24\times 20\times 6 =2880$ is reduced to dimension $274$ by PCA.

%Table \ref{BU3Dtab} presents the experiment results. Since the change of human facial expression has a nice distinction, it is difficult to get high clustering accuracy as shown in Table \ref{BU3Dtab}. [\GaoC{Why it is difficult to have high accuracy when the expression is so distinctive?  Can you check this statement again?  You may want to say something different.}] Even so, the performance of our methods are superior to other methods.

Table \ref{BU3Dtab} presents the experiment results, which show all the methods perform poorly for this challenging dataset.  Analyzing this dataset reveals that some images of different expressions from one person only have very little distinction while some images of the same expression from different persons have a strong distinction, which leads to a difficult problem to find a common feature to present the same expression from different persons. It is not surprised that all the methods perform badly over this dataset. Yet, the performance of manifold based methods are superior to other methods, especially our methods produce a clustering accuracy at least 4 percent better than other methods. 

\begin{table*}
  \centering
   \begin{tabular}{|c|c|c|c|c|c|c|c|c|}
     \hline
              Clusters &GLRR-F &GLRR-21 &KGLRR-ccp &LRR &SSC &SCGSM &SMCE  &LS3C  \\
     \hline
        6         & 0.3900 & 0.3517 & \textbf{0.4033} & 0.3000 & 0.2117 & 0.2583 & 0.3450 & 0.3233 \\
     \hline
   \end{tabular}
  \caption{The clustering results on the BU-3DEF human facial expression database.}\label{BU3Dtab}
\end{table*}

\subsection{Clustering on Large Scale Object Image Sets}
Large scale dataset is challenging for clustering methods. When the cluster number of the data is increasing, the performance of many state-of-the-art  
clustering methods drops dramatically. In this set of experiment, our intention is to test our methods on two large scale object image sets, the Caltech 101 dataset %\cite{LiFergusPerona2007} 
and the Imagenet 2012 dataset. %\cite{DengDongSocherLiLiLi2009}.  %, which both contain more than dozens of classes objects for clustering.

\textit{1) Caltech 101 dataset} %\cite{LiFergusPerona2007}}

Caltech 101 dataset contains pictures of 101 categories objects and each category has about 40 to 800 images. The objects are generally centered in images. All images are grayed and rescaled to a size of $20\times 20$. Fig. \ref{FigE2} shows some samples of Caltech 101 dataset.

In each category, we randomly select $P=4$ images to construct an image set. The image sets are then converted to Grassmann points $X \in \mathcal{G}(4,400)$, i.e. $p=4$.  For both SSC and LRR, the subspace vectors with dimension $20 \times 20 \times 4 = 1600$ are reduced to $249$ by PCA.

To evaluate the robustness of our proposed methods with large cluster numbers, we test the cases of $C = 10, 20, 30, 40$ and $50$. 
Table \ref{Caltechtab} shows the experiment results for different methods. As can be seen from this table, in all cases, our methods outperform other state-of-the-art methods at least 10 percent and are stable with the increase of cluster numbers. It reveals that when the number of clusters is higher, GLRR-F is slightly better than KGLRR-ccp. It is worth noting that our methods are also robust to complex backgrounds contained in Caltech 101 images.

\begin{table*}
  \centering
   \begin{tabular}{|c|c|c|c|c|c|c|c|c|}
     \hline
             Number of Classes &GLRR-F &GLRR-21 &KGLRR-ccp &LRR &SSC &SCGSM &SMCE  &LS3C  \\
     \hline
        10         & 0.8726 & 0.6371 & \textbf{0.9114} & 0.5360 & 0.6066 & 0.5928 & 0.6108  & 0.7230\\
     \hline
        20         & 0.7627 & 0.5144 & \textbf{0.8126} & 0.4290 & 0.3869 & 0.4035 & 0.4246  & 0.5909\\
     \hline
        30         & 0.6664 & 0.4294 & \textbf{0.6995} & 0.3014 & 0.3318 & 0.4550 & 0.3810  & 0.5043\\
     \hline
        40         & \textbf{0.6519} & 0.3654 & 0.6387 & 0.1687 & 0.2881 & 0.3679 & 0.3737  & 0.4617\\
     \hline
        50         & \textbf{0.5913} & 0.3395 & 0.5884 & 0.1451 & 0.2614 & 0.3137 & 0.3608  & 0.4138\\
     \hline
   \end{tabular}
  \caption{The clustering results on the Celtech 101 database.}\label{Caltechtab}
\end{table*}

\textit{2) ImageNet 2012 dataset}  %\cite{DengDongSocherLiLiLi2009}}

ImageNet 2012 dataset is a wildly used object database for image retrieve, which contains more than $1.2$ million object images from over 1000 categories. This database is more difficult for clustering due to the large scale of categories. In addition, most objects are small and un-centered in images, even many objects have severe occlusions.  We extract the region of interest from images by the bounding boxes defined in image-net.org and resize the region to $20\times 20$. Some samples of ImageNet are shown in Fig. \ref{FigE2}.

\begin{figure}
    \begin{center}
    \includegraphics[width=0.45\textwidth]{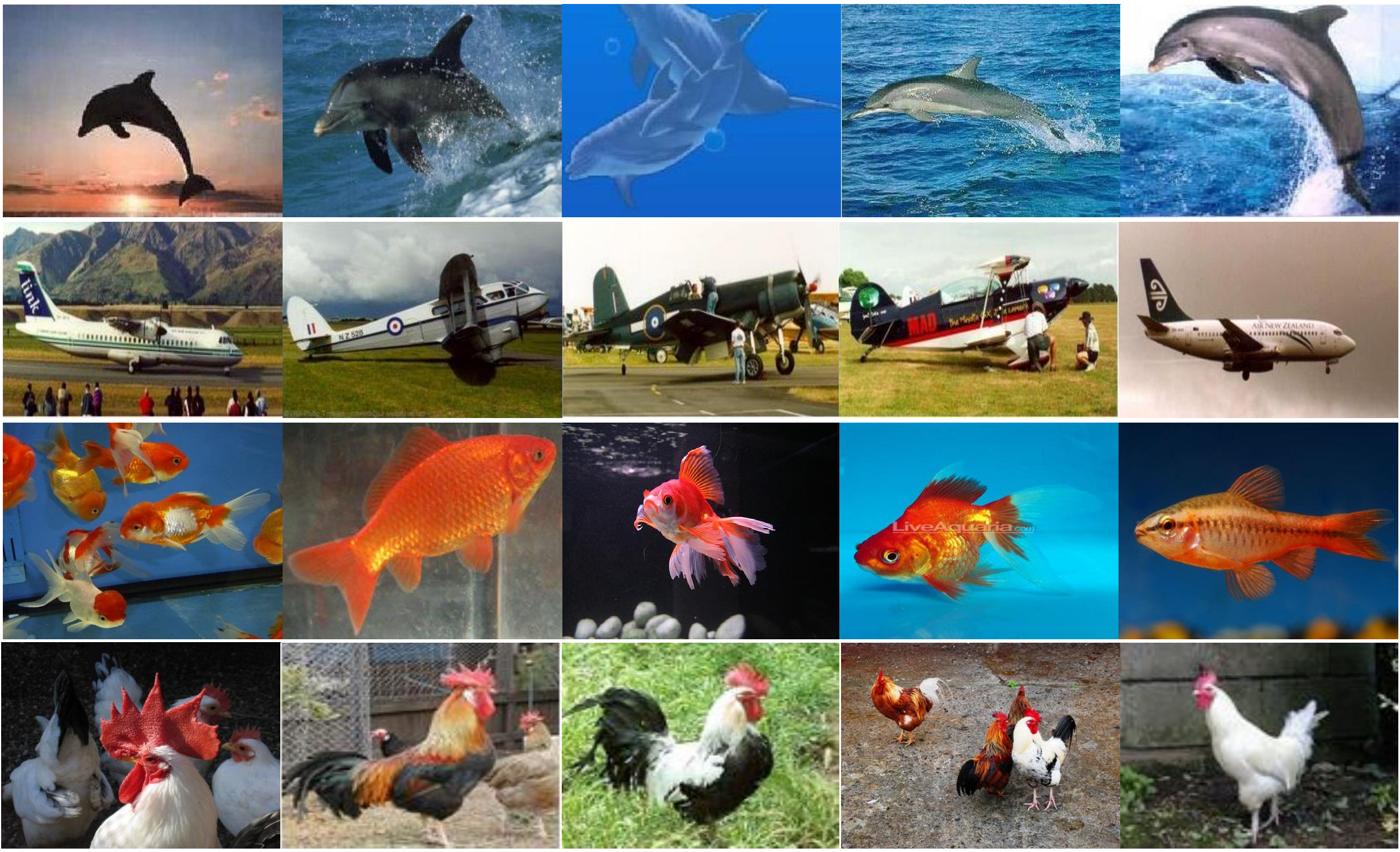}
    \end{center}
    \caption{Some samples of Caltech 101 dataset (the first two rows) and ImageNet 2012 dataset(the last two rows).}\label{FigE2}
\end{figure}

Many categories in this dataset are quite similar to each other, and the orientation and scale of objects in images change largely, so it is hard for an algorithm to get high accuracy for classification or clustering. To test the performance of our methods in the existence of these varieties, we only set a mild value of $C=10$ in this experiments, i.e. we randomly select 10 classes of objects. 
In each class, Grassman points on $\mathcal{G}(2,400)$ are constructed by using SVD for image sets, each of which contains $P=8$ randomly selected images, and a total of 1284 image sets are obtained. For both SSC and LRR, the subspace vector with dimension $20\times 20\times 8 = 3200$ is reduced to dimension $377$ by PCA. 

The clustering experiments are repeated 10 times and the mean clustering accuracy and the variances are reported as the final results. As shown in Table \ref{ImageNettab}, all the comparing methods failed to produce meaningful results, although our methods manage to give the best results among them. There are several factors to make this dataset a great challenge for clustering. First, this dataset collects a larger volume of images with huge number of classes. Second, different categories are more similar to each other than that in Caltech 101 dataset. Third, many objects in this dataset are very small and not well aligned. Finally, the objects in this dataset have complex backgrounds or with occlusions. 
 
\begin{table*}
  \centering
   \begin{tabular}{|c|c|c|c|c|c|c|c|c|}
     \hline
           Statistics &GLRR-F &GLRR-21 &KGLRR-ccp &LRR &SSC &SCGSM &SMCE  &LS3C  \\
     \hline
        mean & 0.3100  & 0.2901 &  \textbf{0.3118}& 0.2243 & 0.1721 & 0.2103 & 0.2812   & 0.2595\\
     \hline
        variance &2.43e-04&8.00e-05&2.13e-04&2.90e-04&4.15e-05&2.48e-04 &3.63e-04 &3.80e-04 \\
     \hline
   \end{tabular}
  \caption{The clustering results on the ImageNet database.}\label{ImageNettab}
\end{table*}

\subsection{Clustering on Human Action Datasets}

Human action classification and recognition is an open and hot issue in computer vision literature. Most human action data are in the form of video clips, which is a set of sequential frame images and suitable to our subspace representation methods. So we select two human action datasets, the Ballet dataset %\cite{WangMori2009} 
and the SKIG dataset, %\cite{liushao2013}, 
to test our clustering methods.

\textit{1) Ballet dataset} % \cite{WangMori2009}}

The Ballet dataset contains 44 video clips collected from an instructional ballet DVD. %\cite{WangMori2009}. 
Each clip has 107 to 506 frames. The dataset consists of 8 complex action patterns performed by 3 subjects. These actions include: `left-to-right hand opening', `right-to-left hand opening', `standing hand opening', `leg swinging', `jumping', `turning', `hopping' and `standing still'. The dataset is challenging due to the significant intra-class variations in terms of speed, spatial and temporal scale, clothing and movement. The frame images are normalized and centered in a fixed size of $20 \times 20$. Some frame samples of Ballet dataset are shown in Fig.~\ref{FigE3}.

%\textbf{\emph{Ballet dataset}},
%we extract every $P=6$ consecutive frames of the same action as an image set and generate about $N = 1444$ image sets.
Similar to the method constructing image sets in \cite{ShiraziHarandiSandersonAlaviLovell2012}, we split each clip into sections of $P=6$ frames to form the image sets. We totally obtain $N=1444$ image sets. The cluster number $C=8$ and the dimension of subspace is set $p=4$. Thus we construct Grassmann points $X\in \mathcal{G}(4,400)$ for clustering. For SSC and LRR methods, the subspace vectors in size of $20\times 20\times 6 = 2400$  is reduced to dimension $160$ by PCA.

The results of different clustering methods are shown in Table \ref{Actiontab}. The ballet images do not have much complex background, thus they can be regarded as clear data without noise. Additionally, the images in each image set have time sequential relations and each action consists of several simple actions. So these helpfully improve the performance of the clustering methods, as shown in Table \ref{Actiontab}. Our methods, SCGSM and SMCE out-stand other methods at least 24 percent, which reflects the advantage of the manifold based methods.  
\begin{comment}
\begin{table*}
  \centering
   \begin{tabular}{|c|c|c|c|c|c|c|c|c|c|}
     \hline
        cluster &GLRR-F &GLRR-21 &KGLRR-ccp &LRR &SSC &SCGSM &SMCE &CGNKE &LS3C  \\
     \hline
        8 & 0.5727 & 0.6122 & 0.6253 & 0.2895 & 0.3089 & 0.5429 & 0.5616 & 0.2486 & 0.2278 \\
     \hline
   \end{tabular}
  \caption{Subspace clustering results on the Ballet database.}\label{Ballettab}
\end{table*}
\end{comment}

\textit{2) SKIG dataset} % \cite{liushao2013}}

SKIG dataset %\cite{liushao2013} 
contains  1080 RGB-D sequences captured by a Kinect sensor. Each RGB-D sequence contains 63 to 605 frames. These sequences have ten kinds of gestures, `circle', `triangle', `up-down', `right-left', `wave', 'Z', `cross', `come-here', `turn-around', and `pat',  performed by six persons.
%: 
All the gestures are performed by fist, finger and elbow, respectively, under three backgrounds
%(wooden board, white plain paper and paper with characters)
and two illuminations.
%(strong light and poor light).
Here the images are normalized to $24\times 32$ with mean zero and unit variance. Fig.~\ref{FigE3} presents some samples of SKIG dataset.
\begin{figure}
    \begin{center}
    \includegraphics[width=0.45\textwidth]{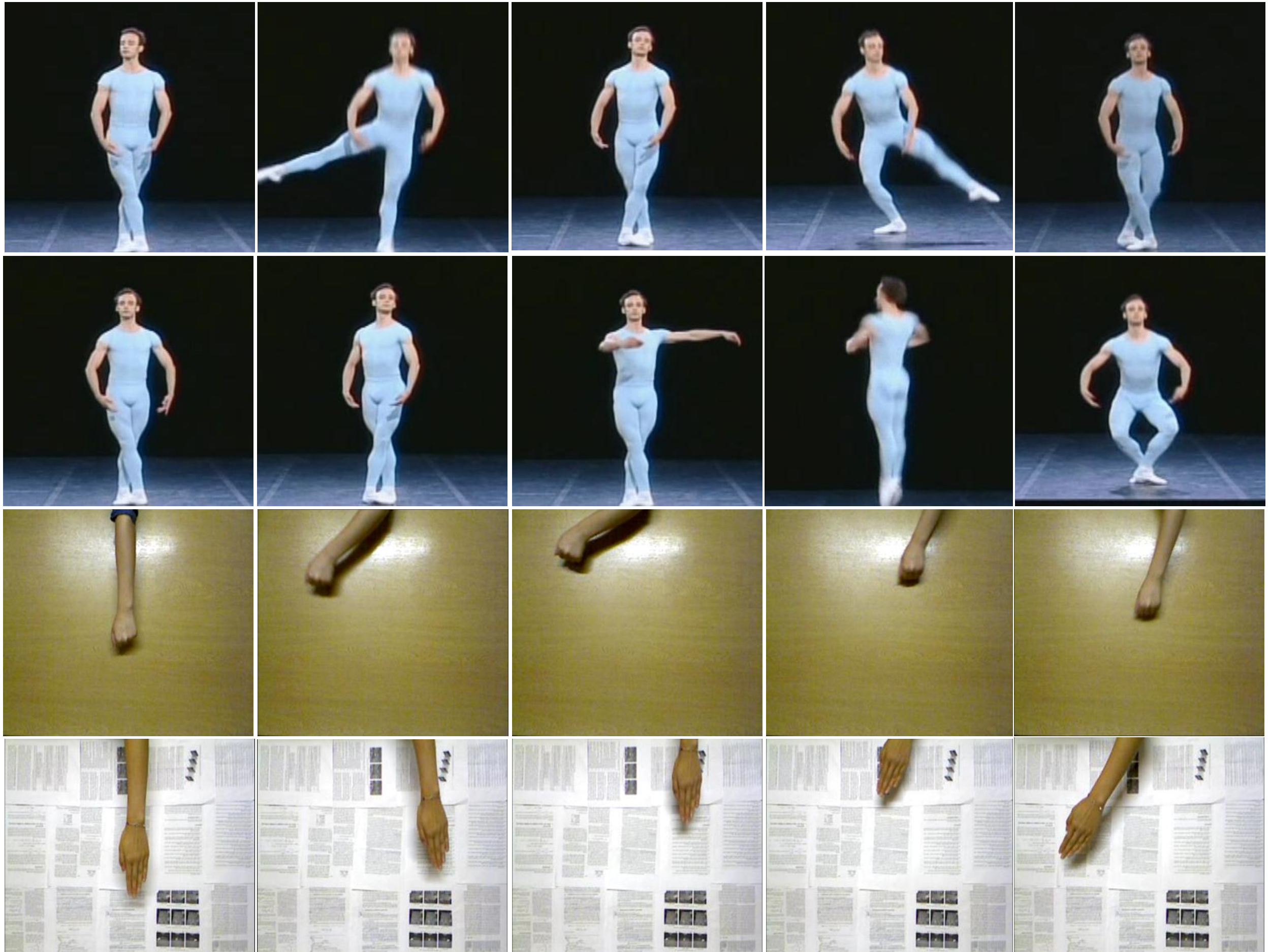}
    \end{center}
    \caption{Some samples of Ballet dataset (the first two rows) and SKIG dataset (the last two rows).}\label{FigE3}
\end{figure}

We regard each RGB video as an image set and obtain totally $N=540$ image sets to cluster into $C=10$ classes. The dimension of subspaces is set $p=20$ and thus the Grassmann points on $\mathcal{G}(20,768)$ are constructed. Since there is a big gap between 63 to 605 frames among SKIG sequences and PCA algorithm requires each sample has equal dimension, it is difficult to select same number of frames for each sequence as the inputs for SSC and LRR. Thus, we give up comparing our methods with SSC and LRR.

%Table \ref{SKIGtab} summarizes testing results in SKIG action database. Even though there are the much various factors of one kind of gesture, our methods get the higher accuracies.
%To test the robustness of our proposed methods, we choose this gesture dataset which contains changes of background and illumination.
This dataset also has more challenges than the ballet dataset, due to the smaller scale of the objects, the various backgrounds and illuminations, however the experimental results in Table \ref{Actiontab} show that our GLRR-F is better than other methods at least 5 percent. 

%Combine with the previous experimental results, our proposed methods present more salient advantages than other methods when the samples exist complex background or illumination changes, which further to demonstrate the robustness of our methods.
\begin{comment}
\begin{table*}
  \centering
   \begin{tabular}{|c|c|c|c|c|c|c|c|}
     \hline
              condition &GLRR-F &GLRR-21 &KGLRR-ccp  &SCGSM &SMCE &CGNKE &LS3C  \\
     \hline
        light + dark    & 0.5185 & 0.4907 & 0.5333 & 0.4667 & 0.4130 & 0.1648 & 0.3722\\
     \hline
   \end{tabular}
  \caption{The clustering results on the Gesture Action database.}\label{SKIGtab}
\end{table*}
\end{comment}

\begin{table}
   \centering
   \begin{tabular}{|c|c|c|}
     \hline
              \diagbox{Methods}{Datasets} & Ballet & SKIG\\
              \hline
              FGLRR & 0.5727 & 0.5185\\
              \hline
              GLRR-21 & 0.6122 &0.4907\\
              \hline
              KGLRR-ccp & \textbf{0.6253} &\textbf{0.5333}\\
              \hline
              LRR & 0.2895 & -\\
              \hline
              SSC & 0.3089 & -\\
              \hline
              SCGSM & 0.5429 & 0.4667\\
              \hline
              SMCE & 0.5616 & 0.4130\\
              \hline
              LS3C & 0.2278 & 0.3722\\
     \hline
   \end{tabular}
  \caption{The clustering results on the Ballet dataset and SKIG dataset.}\label{Actiontab}
\end{table}

\subsection{Clustering on Traffic Scene Video Clips Sets}
The above two human action datasets are considered as having relatively simple scenes with limited backgrounds. To demonstrate robustness of the our algorithms to complex backgrounds, we further test our methods on real traffic scene video clips datasets, including the Highway Traffic Dataset %\cite{ChanVasconcelos2008} 
and a traffic video dataset we collected.

%\subsubsection{Datasets}

\textit{1) Highway Traffic Dataset} % \cite{ChanVasconcelos2008}}

Highway Traffic Dataset %\cite{ChanVasconcelos2008} 
contains 253 video sequences of highway traffic captured under various weather conditions, such as sunny, cloudy and rainy. These sequences are labeled with three traffic levels: light, medium and heavy. There are 44 clips of heavy level, 45 clips of medium level and 164 clips of light level. Each video sequence has 42 to 52 frames. The video sequences are converted to grey images and each image is normalized to size  $24 \times 24$ with mean zero and unit variance. Some samples of the Highway traffic dataset are shown in Fig.~\ref{FigE4}

We regard each video sequence as an image set to construct a point on Grassmann manifold in the similar way used in the above experiments. The subspace dimension is set to $p=20$ and the number clusters equals to the number of traffic levels, i.e. $C=3$. For SSC and LRR, we vectorize the first 42 frames of each clip and then use PCA to reduce the dimension $24\times24\times42=24192$ to 147. Note that there is no clear cut between different levels of traffic jams. For some clips, it is difficult to say whether they belong to heavy, medium or light level. So it is indeed a great challenging task for clustering methods.

Table \ref{Traffictab} presents the clustering performance of all the methods on the Traffic dataset. Because the number of traffic level is only 3, all experimental results seem meaningful and for the worst case the accuracy is 0.5138. GLRR-F's accuracy is 0.8063, which is at least 15 percent higher than other methods. In addition, KGLRR-ccp further improves the clustering accuracy to 0.8221 by the benefit of using kernel methods. An interesting phenomenon is that Euclidean based methods (SSC and LRR) outperform some manifold based methods (SCGSM, SMCE). The reason may be that all frames in a video are similar, in other words, a traffic level can be judged by simply using images, so Euclidean representation could be a better choice.

\begin{comment}
\begin{table*}
  \centering
   \begin{tabular}{|c|c|c|c|c|c|c|c|c|c|}
     \hline
              Clusters &GLRR-F &GLRR-21 &KGLRR-ccp &LRR &SSC &SCGSM &SMCE &CGNKE &LS3C  \\
     \hline
        3         & 0.8063 & 0.5415 & 0.8221 & 0.6838 & 0.6285 & 0.6087 & 0.5138 & 0.5534 & 0.6561 \\
     \hline
   \end{tabular}
  \caption{The clustering results on the Highway Traffic database.}\label{HighwayTraffictab}
\end{table*}
\end{comment}

\textit{2) Our Road Traffic Dataset}

Our road traffic dataset is derived from a total of 300 video clips which were collected from Road detectors used in Beijing. These clips are also labeled as three traffic levels: light, medium and heavy. This database has more scene changes such as `sunny', `cloudy', `heavy rainy' and `darkness'. Each clip consists of over 50 frames. We convert these clips to gray images and resize them to $20\times 20$. Fig.~\ref{FigE4} shows some samples from this dataset.
\begin{figure}
    \begin{center}
    \includegraphics[width=0.45\textwidth]{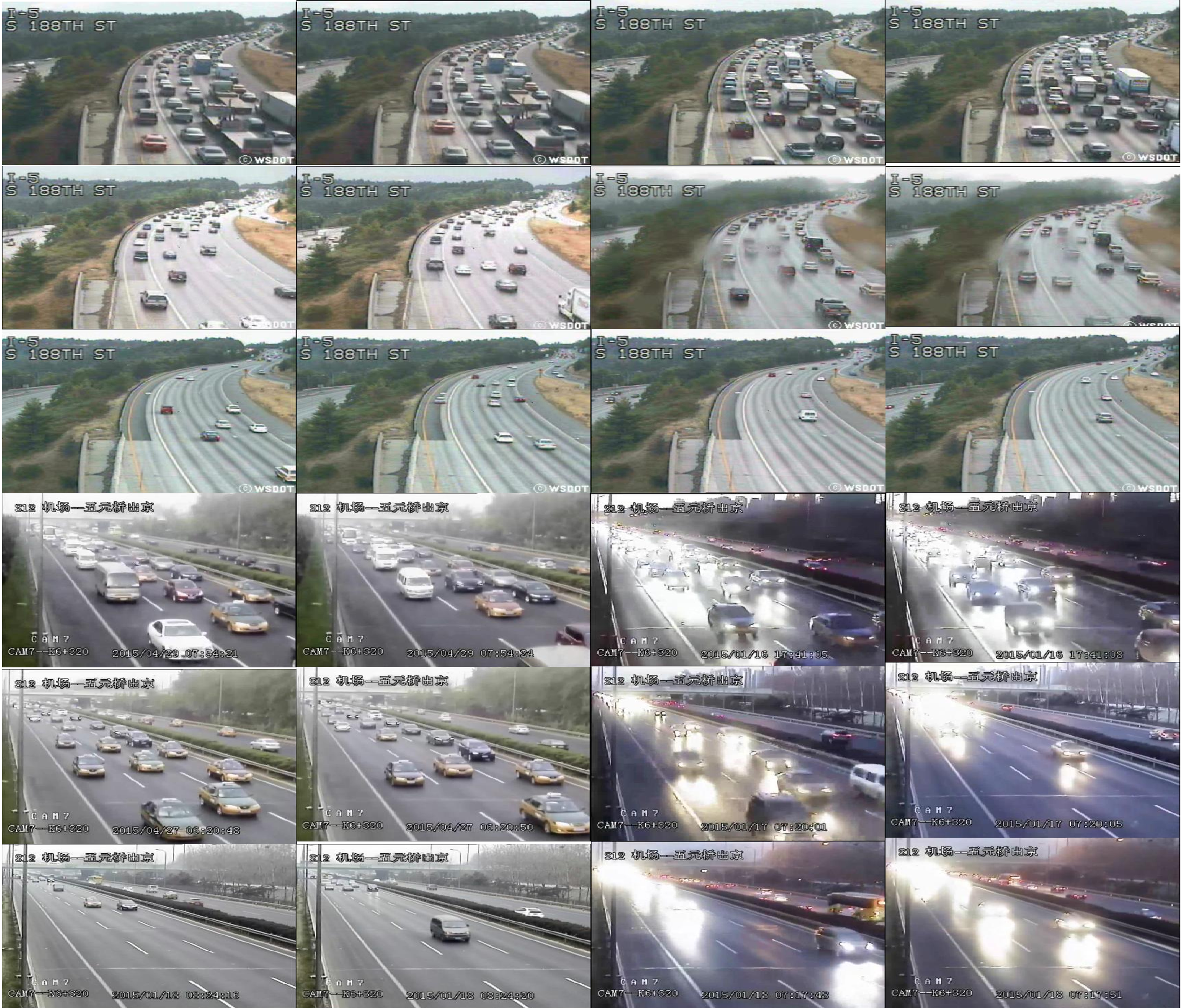}
    \end{center}
    \caption{Some samples of the Highway traffic dataset (the first three rows) and our  Road Traffic Dataset (the last three rows).}\label{FigE4}
\end{figure}

We segment 50 frames from each clip and represent it as an image set. So we have $100$ clips for each traffic level and totally 300 clips for $C=3$ classes. The image sets are also represented as Grassmann points $X\in \mathcal{G}(6,400)$. For SSC and LRR methods, the subspace vector with dimension $20\times20\times50 = 20000$ is reduced to dimension  $41$ by PCA.

The experimental results are listed in Table \ref{Traffictab}. Though the environment in this database is more complex than that in the above traffic database, the accuracy of our methods are obviously at least 4 percent higher than other methods. Once again the experiment on this dataset shows that the Grassmann based methods are more appropriate than other methods for this type of data.  

\begin{comment}
\begin{table*}
  \centering
   \begin{tabular}{|c|c|c|c|c|c|c|c|c|c|}
     \hline
              Clusters &GLRR-F &GLRR-21 &KGLRR-ccp &LRR &SSC &SCGSM &SMCE &CGNKE &LS3C  \\
     \hline
        3         & 0.6778 & 0.6667 & 0.6778 & 0.4911 & 0.6678 & 0.4767 & 0.6656 & 0.5444 & 0.4433 \\
     \hline
   \end{tabular}
  \caption{The clustering results on our Road Traffic database.}\label{RoadTraffictab}
\end{table*}
\end{comment}

\begin{table}
   \centering
   \begin{tabular}{|c|c|c|}
     \hline
              \diagbox{Methods}{Datasets} & Highway Traffic& Road Traffic\\
              \hline
              GLRR-F & 0.8063& \textbf{0.6778}\\
              \hline
              GLRR-21 & 0.5415&0.6667\\
              \hline
              KGLRR-ccp & \textbf{0.8221}&\textbf{0.6778}\\
              \hline
              LRR & 0.6838 & 0.4911\\
              \hline
              SSC & 0.6285 & 0.6678\\
              \hline
              SCGSM & 0.6087 & 0.4767\\
              \hline
              SMCE & 0.5138 & 0.6656\\
              \hline
              LS3C & 0.6561 & 0.4433\\
     \hline
   \end{tabular}\\[1em]
  \caption{The clustering results on the Highway Traffic dataset and our Road Traffic dataset.}\label{Traffictab}
\end{table}

%[\GaoC{Do you think the size $20\times 20$ etc too small?}]

%\subsection{Human Face Clustering}% Sec 5.2

%[\GaoC{4 images or 8 images? Also you need tell these 4 or 8 images are randomly from an object. The current description looks like you can arbitrarily choose 4 or 8 images. This comments also apply to other experiments. Please revise them accordingly. You must tell your reader this image subset is chosen for an object.}]

%[\GaoC{What about the case $P=4$?}]

%The experiment results are reported in Table \ref{Yaletab}. It indicates that SMCE and our methods perform excellently for both cases. When the image number of the image set increasing, the clustering accuracies of our methods have obviously increasing.
%\subsection{Large Scale Images Clustering} % Sec 5.3
%\subsection{Human Action Clustering} % Sec 5.4
%\subsection{Road Traffic Clustering} % Sec 5.5

%===========================================================
\section{Conclusion and Future Work}\label{Sec:6}
In this paper, we proposed novel LRR models on Grassmann manifold by the embedding strategy to construct a metric in terms of Euclidean measure. Two models, GLRR-F and GLRR-21, were proposed to deal with Gaussian noise and non-Gaussian outliers, respectively. A closed-form solution to GLRR-F was presented while an ADMM algorithm was also proposed for GLRR-21. In addition, the LRR model on Grassmann manifold was generalized to its kernelized version under the kernel framework. The proposed models and algorithms were evaluated on several public databases against state-of-the-art clustering algorithms. The experimental results show that the proposed methods outperform the state-of-the-art methods and behave robustly to various change sources. The work has demonstrated that incorporating geometrical property of manifolds via embedding mapping actually facilitates learning on manifold. In the future work, we will focus on the exploring the intrinsic property of Grassmann manifold to construct LRR.

\section*{Acknowledgements}
The research project is supported by the Australian Research Council (ARC) through the grant DP130100364 and also partially supported by National Natural Science Foundation of China under Grant No. 61390510, 61133003, 61370119, 61171169, 61300065 and Beijing Natural Science Foundation No. 4132013.
%===========================================================

\bibliographystyle{IEEEtran}
%\bibliography{egbib_new}
\bibliography{reference_boyue}

%\printbibliography

\begin{IEEEbiography}[{\includegraphics[width=1in,height=1.25in,clip,keepaspectratio]{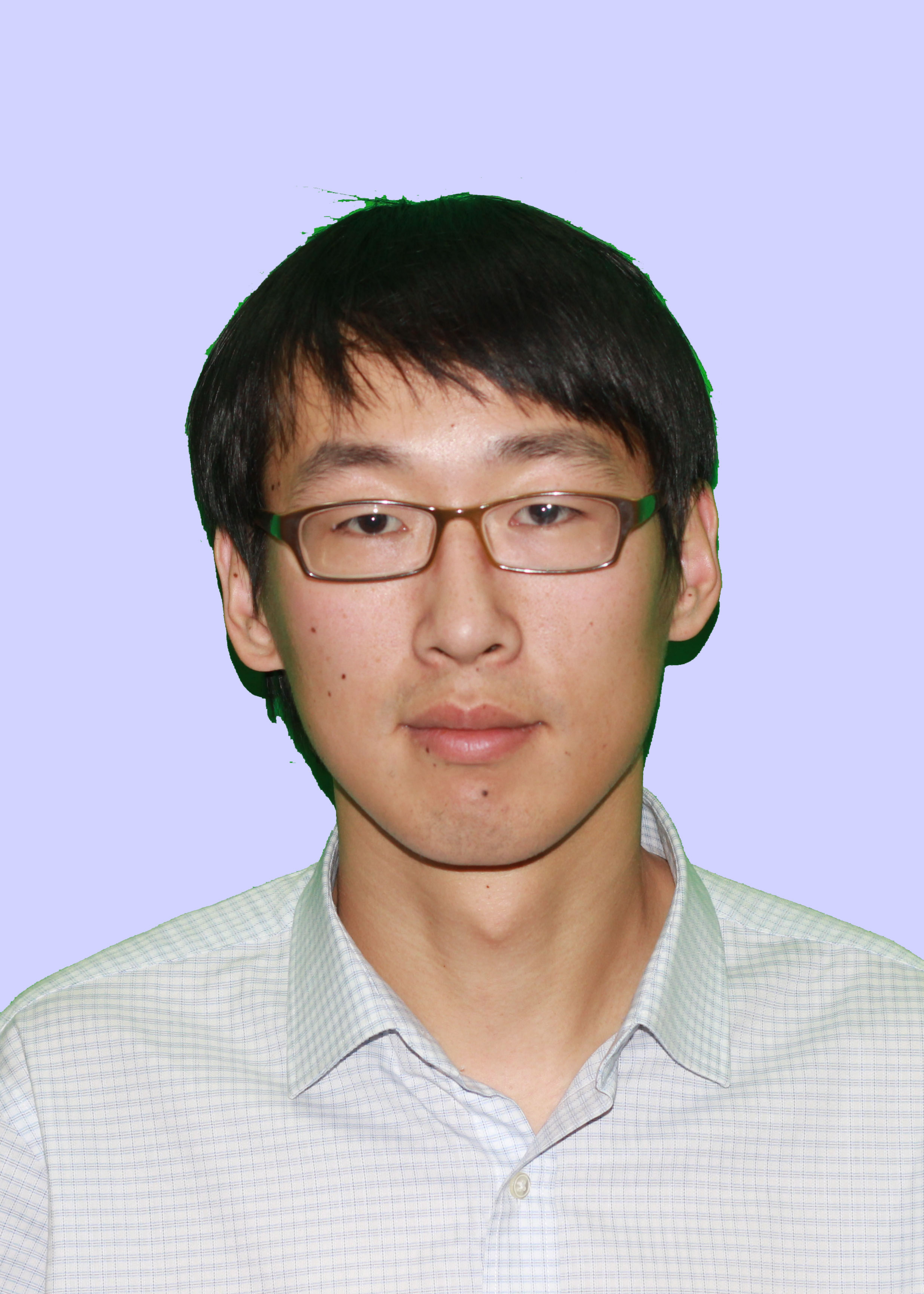}}]
{Boyue Wang} received the B.Sc. degree from Hebei University of Technology,
Tianjin, China, in 2012. he is currently pursuing the
Ph.D. degree in the Beijing Municipal Key Laboratory of Multimedia and Intelligent Software Technology,
Beijing University of Technology, Beijing.
His current research interests include computer
vision, pattern recognition, manifold learning and kernel methods.
\end{IEEEbiography}

\begin{IEEEbiography}[{\includegraphics[width=1in,height=1.25in,clip,keepaspectratio]{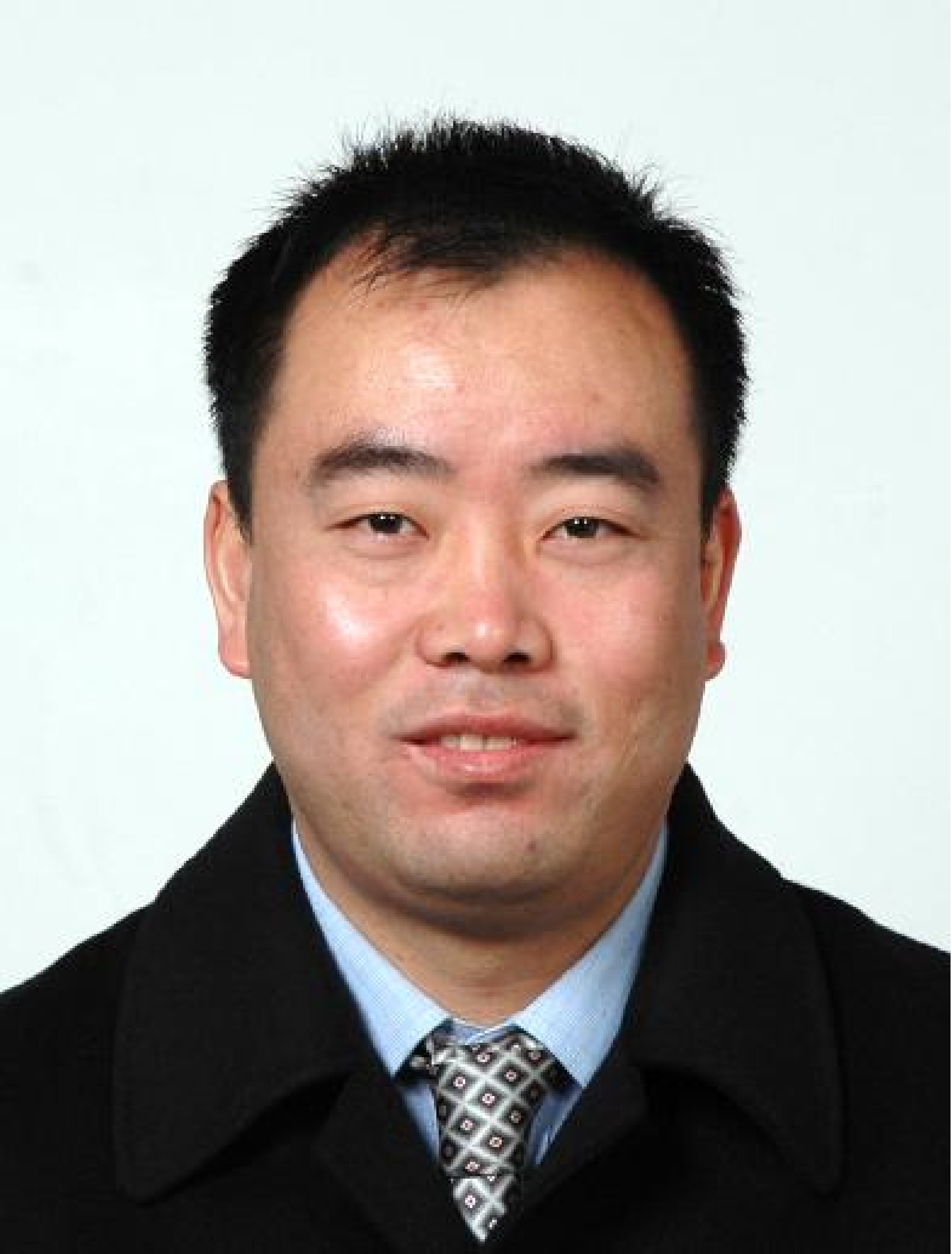}}]
{Yongli Hu} received his Ph.D. degree from Beijing University of Technology in 2005. He is a professor in College of Metropolitan Transportation at Beijing University of Technology. He is
a researcher at the Beijing Municipal Key Laboratory of Multimedia and Intelligent Software Technology.
His research interests include computer graphics, pattern recognition and multimedia technology.
\end{IEEEbiography}

\begin{IEEEbiography}[{\includegraphics[width=1in,height=1.25in,clip,keepaspectratio]{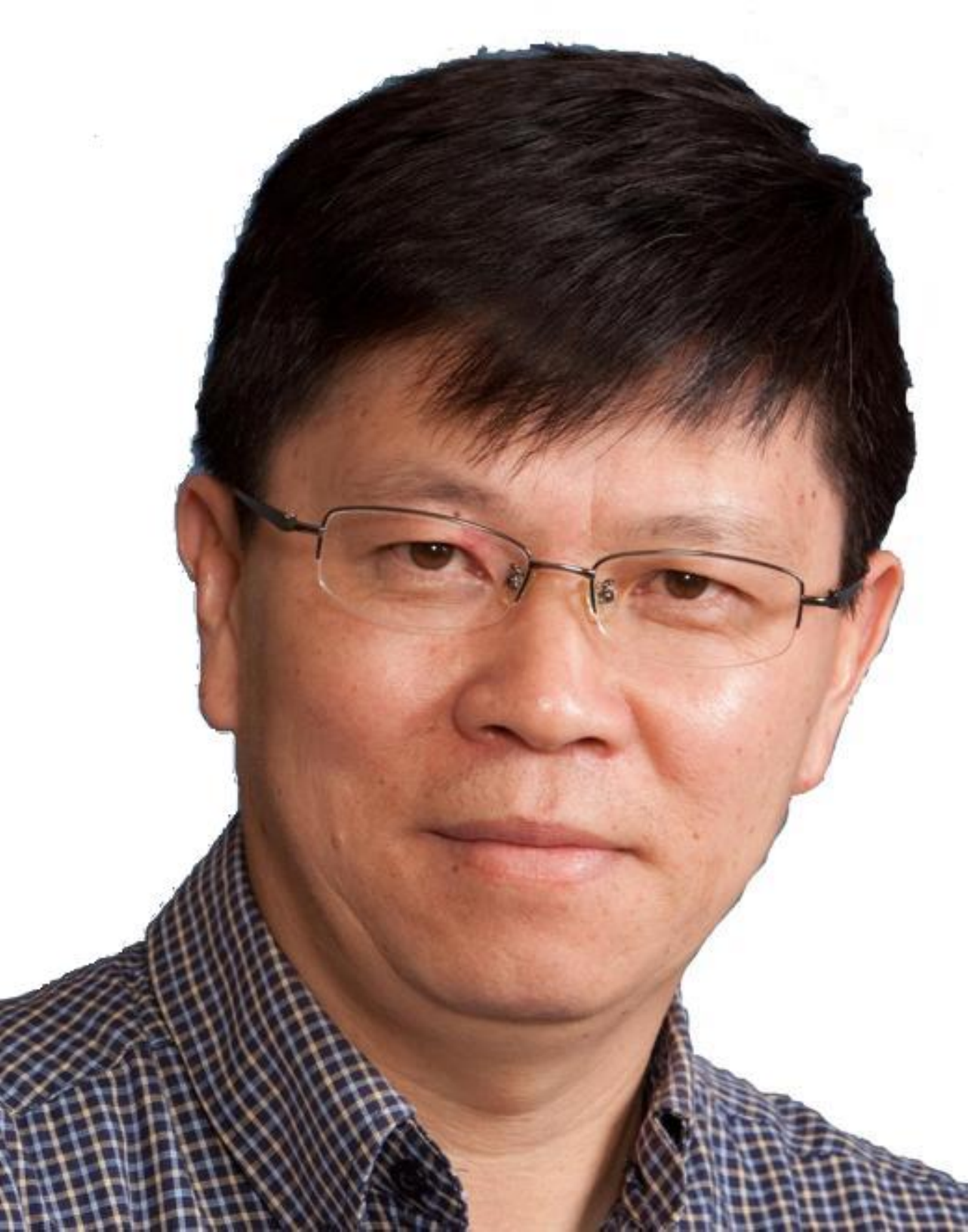}}]
{Junbin Gao} graduated from Huazhong University of Science and Technology (HUST),
China in 1982 with BSc. degree in Computational Mathematics and
obtained PhD from Dalian University of Technology, China in 1991. He is a Professor
in Computing Science in the School of Computing and Mathematics at Charles Sturt
University, Australia. He was a senior lecturer, a lecturer in Computer Science from 2001 to 2005 at
University of New England, Australia. From 1982 to 2001 he was an
associate lecturer, lecturer, associate professor and professor in
Department of Mathematics at HUST. His main research interests
include machine learning, data mining, Bayesian learning and
inference, and image analysis.
\end{IEEEbiography}
 
\begin{IEEEbiography}[{\includegraphics[width=1in,height=1.25in,clip,keepaspectratio]{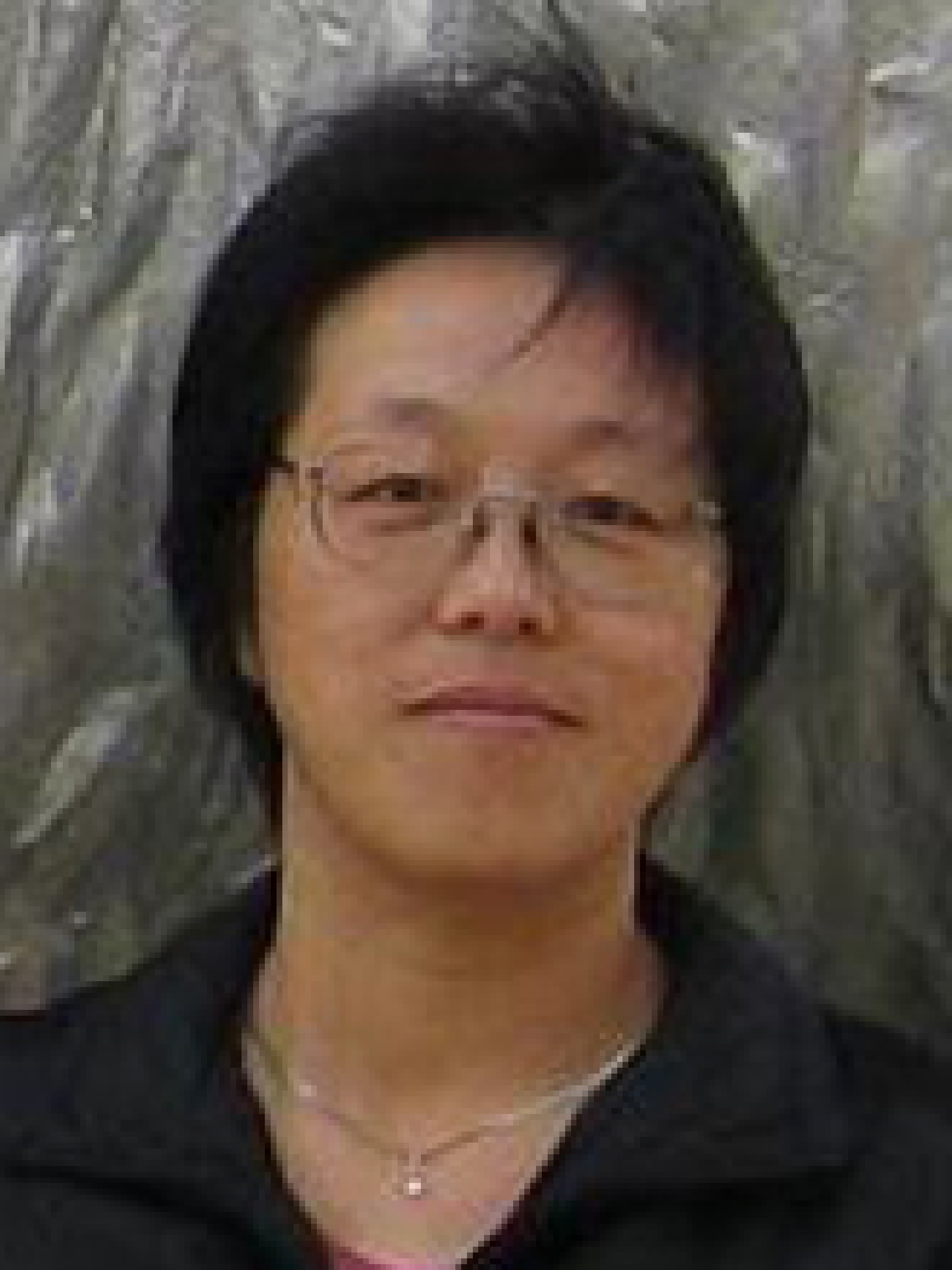}}]
{Yanfeng Sun} received her Ph.D. degree from Dalian University of Technology in 1993. She is a professor in College of Metropolitan Transportation at Beijing University of Technology. She is
a researcher at the Beijing Municipal Key Laboratory of Multimedia and Intelligent Software Technology. She is the membership of China Computer Federation.
 Her research interests are multi-functional perception and image processing.
\end{IEEEbiography}

\begin{IEEEbiography}[{\includegraphics[width=1in,height=1.25in,clip,keepaspectratio]{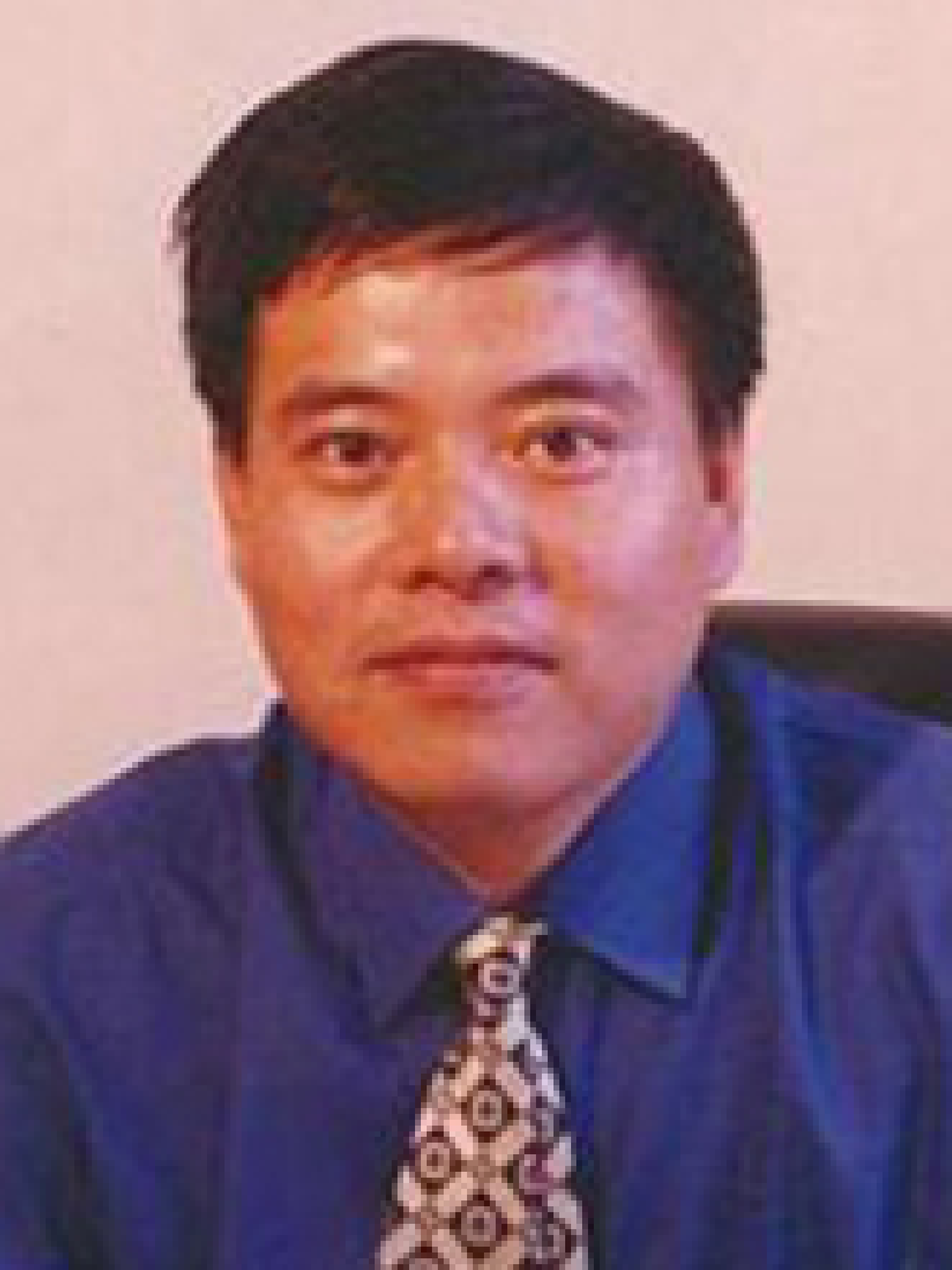}}]
{Baocai Yin} received his Ph.D. degree from Dalian University of Technology in 1993. He is a professor at College of Metropolitan Transportation, Beijing University of Technology. He is
a researcher at the Beijing Municipal Key Laboratory of Multimedia and Intelligent Software Technology. He is the membership of China Computer Federation. His
research interests cover multimedia, multifunctional perception, virtual reality and computer graphics.
\end{IEEEbiography}
\vfill

\end{document}